%% file: main.tex
\documentclass{article}
\usepackage{arxiv}
\usepackage[utf8]{inputenc}
\usepackage[T1]{fontenc}
\usepackage{microtype}
\usepackage[authoryear,round]{natbib}
\usepackage{amsmath,amssymb,amsthm,mathtools}
\usepackage{graphicx,booktabs,array,multirow}
\usepackage{placeins}
\usepackage{algorithm,algpseudocode}
\usepackage{xcolor}
\usepackage{hyperref,url}
\usepackage{tikz}
\usetikzlibrary{arrows.meta,positioning}
\usepackage[most]{tcolorbox}
\newtcolorbox{researchquestions}{
    enhanced,
    colback=black!3,
    colframe=black!65,
    boxrule=0pt,
    leftrule=1.2pt,
    sharp corners,
    left=9pt,
    right=9pt,
    top=7pt,
    bottom=7pt,
    before skip=9pt,
    after skip=9pt,
    fontupper=\normalsize
}
\newcommand{\E}{\mathbb E}
\newcommand{\R}{\mathbb R}
\newcommand{\norm}[1]{\lVert #1\rVert}
\newcommand{\ip}[2]{\langle #1,#2\rangle}
\newcommand{\Reg}{\mathrm{Regret}}
\newcommand{\D}{\mathcal D}
\newcommand{\U}{\mathcal U}

\newtheorem{assumption}{Assumption}
\newtheorem{lemma}{Lemma}
\newtheorem{proposition}{Proposition}
\newtheorem{theorem}{Theorem}

\title{Sample-Adaptive Zeroth-Order Post-Training}
\input{authors}
\date{}
\renewcommand{\shorttitle}{Sample-Adaptive Zeroth-Order Post-Training}
\hypersetup{pdftitle={Sample-Adaptive Zeroth-Order Post-Training},hidelinks}
\begin{document}
\maketitle
\begin{abstract}
% LYY ADD

Zeroth-order (ZO) optimization offers a memory-efficient alternative to backpropagation for fine-tuning large language models. However, existing ZO methods typically sample training examples uniformly, overlooking their evolving importance and leaving avoidable estimation variance. To address this limitation, we propose Sample-Adaptive Zeroth-Order Optimization (SAZO), which learns sample importance online and adaptively allocates queries to accelerate convergence. SAZO extracts gradient-norm information from existing directional differences to estimate sample importance without access to true gradients. Furthermore, to track evolving importance from sparse minibatch feedback, SAZO maintains per-example exponential moving averages that aggregate observations across updates. The method requires no additional model evaluations and only $O(N)$ additional scalar state compared with MeZO. We prove that SAZO asymptotically matches the leading convergence term of the \emph{minimum-variance} sampling oracle,  and achieves a strictly faster convergence rate than standard methods. Extensive experiments across four model families demonstrate improved average downstream performance and up to $6\times$ faster convergence to the same training loss in terms~of~query~count.
\end{abstract}

\input{sections/introduction}
\input{sections/related_work}
\input{sections/method}
\input{sections/theory}
\input{sections/experiments}
\input{sections/conclusion}

\clearpage
\section*{AI use statement}
Generative AI assisted with research organization, literature retrieval, mathematical derivation, manuscript drafting, LaTeX preparation, numerical identity checks, and a separate adversarial review of this draft. It did not generate the reported training measurements.
\bibliography{references}
\bibliographystyle{plainnat}
\clearpage
\appendix
\input{sections/appendix}

\end{document}

%% file: authors.tex
\author{
  Haiyuan Wan\textsuperscript{1}\quad
  Yanyi Li\textsuperscript{2}\quad
  Weicheng Lin\textsuperscript{2}\quad
  Cong Fang\textsuperscript{2}\thanks{Corresponding author:
    \href{mailto:fangcong@pku.edu.cn}{\texttt{fangcong@pku.edu.cn}}.}\\[4pt]
  \normalfont\small\textsuperscript{1}School of Physics, Peking University\\
  \normalfont\small\textsuperscript{2}School of Intelligence Science and Technology, Peking University
}
\hypersetup{pdfauthor={Haiyuan Wan, Yanyi Li, Weicheng Lin, Cong Fang}}

%% file: sections/introduction.tex
\section{Introduction}
% LYY ADD
Fine-tuning adapts pre-trained large language models (LLMs) to diverse downstream tasks, including sentiment analysis, natural language inference, and question answering \citep{devlin2019bert}. However, standard fine-tuning incurs substantial memory overhead from stored activations, gradients, and optimizer states \citep{li2026prac}, requiring up to $12\times$ the inference memory in the settings evaluated by \citet{malladi2023mezo}. Parameter-efficient fine-tuning (PEFT), including prefix tuning \citep{li2021prefix} and low-rank adaptation (LoRA) \citep{hu2021lora}, reduces trainable parameters but still stores activations for backpropagation. To overcome this remaining activation-memory bottleneck, zeroth-order (ZO) methods such as MeZO bypass backpropagation and estimate updates from two forward evaluations, achieving a memory footprint close to inference \citep{malladi2023mezo}. This memory advantage comes with an optimization challenge: the resulting gradient estimates remain noisy, limiting the convergence speed of ZO fine-tuning \citep{nesterov2017random}.

Inspired by first-order adaptive optimizers such as Adam \citep{kingma2015adam}, mainstream approaches to accelerating ZO optimization focus on adaptation in parameter space. HiZOO \citep{zhao2025hizoo} and PaZO \citep{zhao2026pazo} use Hessian-informed preconditioning, ZO-AdaMM \citep{chen2019zoadamm} and R-AdaZO \citep{shu2025refining} adapt coordinate-wise update scales using historical moment estimates, and LOZO \citep{chen2025lozo} exploits low-rank perturbation subspaces. However, ZO methods face a fundamental information bottleneck: each directional query reveals only a scalar projection rather than a full gradient. Identifying an unrestricted $d$-dimensional adaptation vector requires $\Omega(d)$ such observations, so the cost of learning reliable parameter-wise adaptation can offset its intended acceleration. We therefore ask:
% Mainstream approaches to accelerating ZO optimization focus on adaptation in parameter space. HiZOO \citep{zhao2025hizoo} and PaZO \citep{zhao2026pazo} use Hessian-informed preconditioning, ZO-AdaMM \citep{chen2019zoadamm} and R-AdaZO \citep{shu2025refining} adjust coordinate-wise update scales using historical moment estimates, and LOZO \citep{chen2025lozo} exploits low-rank perturbation subspaces. Yet favorable parameter-space geometry alone does not ensure low-variance stochastic updates, leaving a source of slow convergence that these geometric improvements do not fully address. We therefore ask:
\begin{researchquestions}
\noindent
\textbf{A.}\enspace
\textit{Is there additional information in ZO queries that current optimizers leave untapped?}

\par\medskip
\noindent
\textbf{B.}\enspace
\textit{If so, how can we adaptively exploit this information to accelerate ZO convergence?}
\end{researchquestions}

\begin{figure}[t]
\centering
\includegraphics[width=\linewidth]{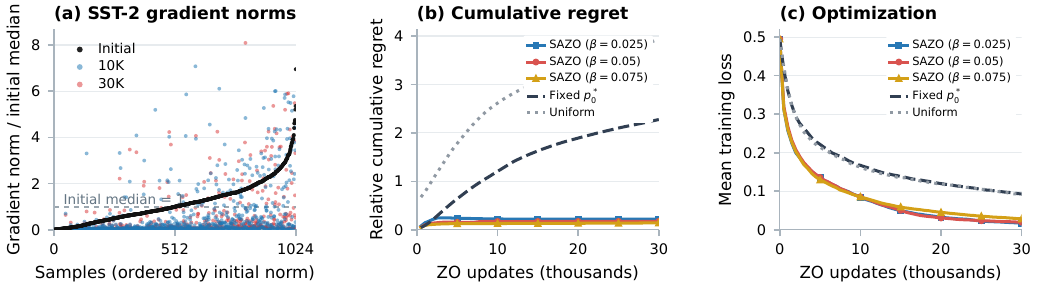}
\caption{\textbf{From sample heterogeneity to allocation regret and optimization.} (a) Per-example gradient norms on SST-2 at initialization, 10K, and 30K updates, normalized by the initial median. Points retain the sample order induced by the initial norms, showing both cross-sample heterogeneity and temporal changes. (b) Sampling regret is the excess gradient-estimator variance over the variance-minimizing allocation (see Appendix~\ref{app:regret-settings} for definitions and detailed settings); SAZO reduces this quantity versus uniform and fixed initial-oracle sampling. (c) The improved allocation accelerates loss reduction.}
\label{fig:intro-narrative}
\end{figure}

Since each update uses only a small subset of the training data, its effectiveness depends in part on which examples are selected. Our empirical analysis identifies an untapped source of information: the evolving importance of individual training examples. Their gradient norms differ substantially and change during training (Figure~\ref{fig:intro-narrative}(a); Appendix~\ref{app:heterogeneity}), suggesting that queries need not be equally valuable across~examples. Uniform sampling overlooks this heterogeneity, leaving avoidable
estimation variance. To use this information for query allocation, a natural approach is importance sampling, inspired by first-order methods \citep{katharopoulos2018importance,borsos2018vrb,elhanchi2020avare}: allocate more queries to examples with larger gradient norms. Although sampling heterogeneity also exists in first-order methods, we argue that zeroth-order algorithms are more sensitive to it, since the estimated gradient variance of zeroth-order methods using only function-value access is much larger than that of first-order methods\footnote{The variance is actually $\Omega(d)$ times larger in the worst case.}.

However, applying importance sampling principle to ZO fine-tuning presents two challenges.
\textbf{First, sample importance is not directly observable.} The variance-minimizing allocation depends on per-example gradient norms, $p_{i,t}^{\star}\propto\|\nabla f_i(x_t)\|$, whereas ZO queries reveal only random directional measurements. We therefore use the scaled magnitudes of existing directional differences to estimate these norms; their expectations approach the corresponding gradient norms as the perturbation radius vanishes, requiring no additional queries. \textbf{Second, evolving importance must be tracked from sparse minibatch feedback.} Each update reveals importance information only for examples in the sampled minibatch, while the importance of all examples may change as the model evolves. We therefore maintain per-example exponential moving averages to combine newly observed information with historical estimates, allowing importance tracking across intermittent observations. Uniform exploration provides opportunities to refresh infrequently sampled examples.

Together, these designs yield \emph{Sample-Adaptive Zeroth-Order Optimization (SAZO)}. The learned scores guide sampling, while inverse-probability weighting
preserves the expected full-data finite-difference update.
SAZO requires no additional model evaluations and only $O(N)$ scalar state for $N$ training examples. The benefit of tracking rather than fixing this allocation is visible in Figure~\ref{fig:intro-narrative}(b): SAZO incurs lower relative cumulative sampling regret than uniform sampling and an oracle allocation fixed at initialization. This improved allocation also translates into faster training-loss reduction in Figure~\ref{fig:intro-narrative}(c), connecting the recovered sample information to optimization progress.

Our main contributions are:
\begin{itemize}
    \item We propose SAZO, which estimates sample importance from existing directional differences without extra model evaluations and tracks its evolution from sparse minibatch feedback using per-example exponential moving averages, requiring only $O(N)$ scalar state. 

    \item We prove that SAZO asymptotically matches the leading convergence term of the minimum-variance sampling oracle, and achieves a strictly faster convergence rate than standard methods.

    \item We conduct extensive experiments across four model families, demonstrating improved average downstream performance at MeZO-like memory and per-update cost. SAZO achieves up to $6\times$ query-count speedup over MeZO at the same training loss.
\end{itemize}

%% file: sections/related_work.tex
\section{Related work}\label{sec:related}
\paragraph{Sample-adaptive first-order optimization.}
Online variance minimization learns importance distributions from sampled feedback \citep{bouchard2016sample,namkoong2017adaptive,salehi2017bandit}. Static-regret analyses treat feedback as general loss sequences \citep{borsos2018vrb,borsos2019mixtures}; subsequent work connects sampling to Adam \citep{liu2020adambs}, federated convergence \citep{zeng2024kvib,zhao2025clientsampling}, and joint convex optimization \citep{aolaritei2025optimal}. Avare exploits gradient continuity for dynamic-regret guarantees \citep{elhanchi2020avare}. SAZO addresses the ZO setting, where gradients are unavailable and directional feedback remains noisy, while accounting for the coupling between sampling and optimization.

\paragraph{Accelerating ZO post-training.}
MeZO establishes memory-efficient language-model fine-tuning using forward evaluations \citep{malladi2023mezo}. Subsequent approaches accelerate optimization through curvature-based preconditioning \citep{zhao2025hizoo,zhao2026pazo}, low-rank perturbations \citep{chen2025lozo}, variance reduction \citep{gautam2024variance}, and directional alignment \citep{ma2025dap}. Other methods improve update normalization \citep{dang2026fzoo}, memory-efficient optimizer adaptivity \citep{dbouk2026meazo}, or adaptive layer selection \citep{wang2026adalezo}. These studies focus on constructing parameter updates and allocating perturbations within the model, where learning reliable parameter-wise adaptation is intrinsically difficult in high dimensions; SAZO instead adapts query allocation across~examples.

%% file: sections/method.tex
\section{Sample-adaptive zeroth-order optimization}\label{sec:method}
Consider the finite-sum problem $\min_{x\in\R^d} F(x)=N^{-1}\sum_{i=1}^N f_i(x)$, where $f_i$ denotes the loss of training example $i$. Given a uniformly sampled index $I_t$ and an independent Gaussian direction $z_t\sim\mathcal N(0,I_d)$, ZO optimization uses the finite difference
\begin{equation}\label{eq:fd}
s_{i,t}=\frac{f_i(x_t+\mu_tz_t)-f_i(x_t-\mu_tz_t)}{2\mu_t}.
\end{equation}
to update $x_{t+1}=x_t-\eta s_{I_t,t}z_t$, with perturbation radius $\mu_t>0$ and step size $\eta>0$. Since $\E[z_tz_t^\top]=I_d$, the estimator satisfies $\E[s_{I_t,t}z_t\mid x_t]=\nabla F(x_t)+O(\mu_t^2)$ and is unbiased up to finite-difference error.
Let $g_{i,t}=\nabla f_i(x_t)$ and $a_{i,t}=\norm{g_{i,t}}$, and let $\E_t$ denote expectation conditional on the history before the current draws. The estimator's second moment~is
\begin{equation}\label{eq:uniform-second}
\E_t\norm{s_{I_t,t}z_t}^2=\frac{d+2}{N}\sum_{i=1}^N a_{i,t}^2+O(\mu_t^2).
\end{equation}
The $O(\mu_t^2)$ terms account for finite-difference error under the Hessian regularity conditions in Assumption~\ref{ass:smooth}; their coefficients may depend on the sampling probabilities and gradient norms (Appendix~\ref{app:fd}). The maximum perturbation radius is denoted by $\mu=\max_{t\le T}\mu_t$.

\subsection{Sample-adaptive updates}\label{sec:sample-adaptive-updates}
The update above assigns equal query weight to every example, as does standard first-order SGD \citep{katharopoulos2018importance}. Yet empirical evidence reveals pronounced sample heterogeneity in post-training: gradient norms vary substantially across examples and evolve during optimization (Figures~\ref{fig:intro-narrative}(a) and~\ref{fig:importance-dynamics}). This heterogeneity exhibits interpretable sample-level structure that cannot be explained simply by label class, sequence length, or vocabulary (Appendices~\ref{app:heterogeneity-controls} and~\ref{app:sample-cases}). These observations motivate SAZO (Algorithm~\ref{alg:single}), which uses importance sampling to assign different query weights across examples to reduce update variance and accelerate optimization.

\begin{algorithm}[t]
\caption{Sample-Adaptive Zeroth-Order Optimization}\label{alg:single}
\begin{algorithmic}[1]
\Require $x_1$, $m_{i,1}>0$, step budget $T$, radii $\mu_t>0$, learning rate $\eta>0$
\Require Exploration $\varepsilon\in(0,\tfrac12]$ and smoothing coefficient $\beta\in(0,1)$
\For{$t=1,\ldots,T$}
\State $p_{i,t}\gets(1-\varepsilon)\dfrac{m_{i,t}}{\sum_jm_{j,t}}+\dfrac{\varepsilon}{N}$, for $i=1,\ldots,N$. \Comment{Importance sampling}
\State Draw $I_t\sim p_t$ and fresh $z_t\sim\mathcal N(0,I_d)$.
\State $s_{I_t,t}\gets\dfrac{f_{I_t}(x_t+\mu_tz_t)-f_{I_t}(x_t-\mu_tz_t)}{2\mu_t}$.
\State $x_{t+1}\gets x_t-\eta\dfrac{s_{I_t,t}}{Np_{I_t,t}}z_t$. \Comment{Unbiased gradient estimator}
\State Copy $m_{t+1}\gets m_t$.
\State $m_{I_t,t+1}\gets(1-\beta)m_{I_t,t}+\beta\sqrt{\dfrac\pi2}|s_{I_t,t}|$. \Comment{Sparse update with EMA smoothing}
\EndFor
\end{algorithmic}
\end{algorithm}

% Temporarily excluded; artwork and source remain in figures/.
% \begin{figure}[t]
% \centering
% \includegraphics[width=\linewidth]{figures/sazo_method_overview.pdf}
% \caption{Structure of the of SAZO. The inner importance-tracking loop reallocates queries for the outer optimization loop, whose model updates in turn change sample importance.}\label{fig:loop}
% \end{figure}

Let $p_{i,t}>0$ denote the sampling probability of example $i$, with $\sum_i p_{i,t}=1$. For $I_t\sim p_t$ drawn independently of $z_t$, the importance-weighted update is
\begin{equation}\label{eq:updates}
\widehat g_t=\frac{s_{I_t,t}}{Np_{I_t,t}}z_t,
\qquad x_{t+1}=x_t-\eta\widehat g_t.
\end{equation}
Importance weighting preserves the expected full-data ZO update, with $\E_t\widehat g_t=\nabla F(x_t)+O(\mu_t^2)$. Its second moment~is
\begin{equation}\label{eq:method-second}
\E_t\norm{\widehat g_t}^2
=\frac{d+2}{N^2}\sum_{i=1}^N\frac{a_{i,t}^2}{p_{i,t}}+O(\mu_t^2).
\end{equation}
Under uniform sampling, $p_{i,t}=1/N$, this second moment reduces to Eq.~\eqref{eq:uniform-second}. For positive gradient norms, the distribution that minimizes the leading term and the resulting second moment are
\begin{equation}\label{eq:method-oracle}
p_{i,t}^{\star}=\frac{a_{i,t}}{\sum_j a_{j,t}},
\qquad
\left.\E_t\norm{\widehat g_t}^2\right|_{p_t=p_t^\star}
=\frac{d+2}{N^2}\left(\sum_{i=1}^N a_{i,t}\right)^2+O(\mu_t^2).
\end{equation}

\subsection{Estimating sample importance from directional differences}
The ideal allocation in Eq.~\eqref{eq:method-oracle} requires the current per-example gradient norms. As identified in the Introduction, SAZO addresses two obstacles to learning them in ZO optimization.

\textbf{Unobserved gradients.} ZO optimization cannot access $a_{i,t}=\norm{\nabla f_i(x_t)}$, but the directional difference $s_{i,t}$ already computed for the model update is approximately a one-dimensional Gaussian projection of the gradient. Its scaled absolute value therefore gives a gradient-norm estimator that is unbiased up to finite-difference error, without additional~queries:
\begin{equation}\label{eq:method-observation}
Y_{i,t}=\sqrt{\frac\pi2}|s_{i,t}|,
\qquad \E_tY_{i,t}=a_{i,t}+O(\mu_t^2).
\end{equation}

\textbf{Sparse feedback.} Each update reveals $Y_{i,t}$ only for the sampled example, even though the importance of every example may evolve as optimization proceeds. SAZO therefore maintains a separate exponential moving average for every example, carrying historical estimates across the intervals between observations:
\begin{equation}\label{eq:ema}
m_{i,t+1}=\begin{cases}
(1-\beta)m_{i,t}+\beta Y_{i,t},&i=I_t,\\
m_{i,t},&i\ne I_t.
\end{cases}
\end{equation}
A larger $\beta$ responds faster to changing importance, while a smaller $\beta$ places more weight on accumulated observations and suppresses directional noise.

Finally, SAZO normalizes these smoothed estimates and adds uniform exploration, following the exploration--exploitation principle in EXP3 \citep{auer2002nonstochastic}:
\begin{equation}\label{eq:prob}
p_{i,t}=(1-\varepsilon)\frac{m_{i,t}}{\sum_jm_{j,t}}+\frac\varepsilon N.
\end{equation}
Exploration refreshes stale scores, while the learned distribution guides the next queries. This creates a coupled estimation--optimization loop, analyzed under Assumption~\ref{ass:smooth} in Section~\ref{sec:theory}.

\textbf{Query and storage cost.} SAZO reuses the directional differences already computed for the model update and therefore requires no additional ZO queries. It stores only one scalar importance estimate per example, for $O(N)$ additional space; with 1,000 training examples and FP64 scores, this is 8,000 bytes, or less than 8 KiB, and is negligible relative to the inference-memory footprint measured in Section~\ref{sec:efficiency}.

\textbf{Initialization and minibatch implementation.} We initialize each $m_i$ by averaging its directional-magnitude observations over four complete, stable-order burn-in passes without replacement, using ordinary MeZO updates; these updates and queries count toward the training budget. We then switch to adaptive sampling with replacement and EMA tracking. Each minibatch shares one perturbation direction, as in MeZO \citep{malladi2023mezo}; repeated draws are evaluated once, with their multiplicities retained in the weighted update (Appendix~\ref{app:batch}).

\textbf{Recommended hyperparameters.} We recommend $\beta=0.05$ and $\varepsilon=0.1$, which are used for all main SAZO results. This single setting performs robustly across tasks and datasets and across masked and autoregressive models of different scales; the controlled sweeps in Section~\ref{sec:beta-sensitivity} further show low sensitivity over broad ranges of both parameters.

%% file: sections/theory.tex
\subsection{Convergence analysis}\label{sec:theory}
This section analyzes convergence under the following continuity assumptions. Smoothness conditions of this type are standard in zeroth-order optimization and are routinely used to analyze ZO fine-tuning of language models \citep{nesterov2017random,malladi2023mezo}. They are also well matched to our setting: Lipschitz gradients enable descent and stable tracking as the parameters evolve, while Lipschitz Hessians control the error of SAZO's symmetric finite~differences.

\begin{assumption}\label{ass:smooth}
Each $\nabla f_i$ is globally $L_i$-Lipschitz, and $F$ is $L_F$-smooth and bounded below by $F_*$. Each Hessian is globally $H_i$-Lipschitz. The positive radii $\mu_t$ are deterministic, and the initialization and finite-time moments appearing below are finite. Let $L_2^2=N^{-1}\sum_i L_i^2$.
\end{assumption}

At the step-$t$ parameters, ignoring finite-difference error, the optimal allocation has second moment $V_t^\star=\frac{d+2}{N^2}\left(\sum_{i=1}^N a_{i,t}\right)^2$. Under the sampling distribution $p_t$ in Algorithm~\ref{alg:single}, the corresponding second moment~is $V_t(p_t)=\frac{d+2}{N^2}\sum_{i=1}^N\frac{a_{i,t}^2}{p_{i,t}}$. Their cumulative difference,
\begin{equation}\label{eq:regret-definition}
\Reg_T=\sum_{t=1}^T\E\left[V_t(p_t)-V_t^\star\right]\ge0,
\end{equation}
satisfies the following bound.
\begin{theorem}[Sampling regret]\label{thm:tracking}
Under Assumption~\ref{ass:smooth}, run Algorithm~\ref{alg:single} with fixed $x_1$ and constant $\eta,\varepsilon,\beta$ within the horizon. Let $\mu=\max_{t\le T}\mu_t\le1$, $\sigma_G^2=\pi/2-1$, $E_0=\E\norm{m_1-a_1}^2$, $r=\beta\varepsilon/N$, $\Lambda_d=(d+2)L_2^2$, and
\begin{equation}\label{eq:complexities}
A_T=\frac1T\sum_{t=1}^T\sum_i\E a_{i,t}^2,\qquad
\D_T=\frac1T\sum_{t=1}^T\E V_t^\star.
\end{equation}
If $\kappa=192\Lambda_d\eta^2/(\varepsilon r^2)<1$,~then
\begin{equation}\label{eq:main-regret}
\frac{\Reg_T}{T}\le\frac1{1-\kappa}\left[
\frac{16(d+2)E_0}{TN\varepsilon r}
+\frac{24(d+2)\sigma_G^2\beta^2}{N\varepsilon r}A_T
+\left(\kappa+4\varepsilon\right)\D_T\right]+O(\mu^4).
\end{equation}
\end{theorem}
The bound separates initialization, observation noise, target drift, and exploration, with finite differences contributing an additive remainder. The condition $\kappa<1$ requires the feedback amplification of tracking error through sampling and parameter updates to be weaker than the error contraction induced by the EMA.
Theorem \ref{thm:tracking} shows that our importance tracking method can asymptotically achieve the minimum sampling variance. 

Substituting this result into the descent lemma \citep[Lemma~1.2.3]{nesterov2004introductory} yields the following convergence bound.

\begin{theorem}[Convergence]\label{cor:rate}
Under Assumption~\ref{ass:smooth}, fix $N,d,L_i$ and positive schedule coefficients; use
\[
\eta=c_\eta T^{-1/2},\qquad
\varepsilon=c_\varepsilon T^{-1/10},\qquad
\beta=c_\beta T^{-3/10}.
\]
Let $\Delta_F=F(x_1)-F_*$ and define $K_T=\frac{(d+2)\sigma_G^2c_\beta}{c_\varepsilon^2}A_T+\left(\frac{N^2\Lambda_dc_\eta^2}{c_\varepsilon^3c_\beta^2}+c_\varepsilon\right)\D_T$.
For $\Delta_F,\D_T>0$, take $c_\eta=\sqrt{2\Delta_F/(L_F\D_T)}$, which minimizes the leading expression for fixed $\D_T$. If the parameters are feasible and $\kappa\le1/2$,~then
\begin{equation}\label{eq:rate-explicit}
\frac1T\sum_{t=1}^T\E\norm{\nabla F(x_t)}^2\le\sqrt{\frac{2L_F\Delta_F\D_T}{T}}
+O\left(\frac{L_Fc_\eta K_T}{T^{3/5}}+\frac{L_Fc_\eta(d+2)E_0}{c_\varepsilon^2c_\beta T}\right)+O(\mu^2).
\end{equation}
The first $O(\cdot)$ hides only universal numerical constants; finite-difference coefficients are detailed in Appendix~\ref{app:fd}.
\end{theorem}
The first term is the leading optimization term; under fixed problem and schedule coefficients and uniformly bounded trajectory moments, the adaptation correction is $o(T^{-1/2})$. Here $\D_T$ is the expected horizon average of $V_t^\star$ defined above. The choice of $c_\eta$ is conditional on this trajectory-dependent quantity; proofs and parameter qualifications appear in Appendices~\ref{app:proofs} and~\ref{app:balance}.

Table~\ref{tab:oracle-comparison} compares the leading convergence terms for uniform sampling (MeZO), the best fixed allocation in hindsight, and the optimal allocation at each step, which governs the SAZO bound. With $p^{\mathrm{unif}}=(1/N,\ldots,1/N)$ and $\Delta_N$ the probability simplex, define
\[
\U_T=\frac1T\sum_t\E V_t(p^{\mathrm{unif}})=\frac{d+2}{N}A_T,\qquad
\mathcal H_T=\E\left[\inf_{p\in\Delta_N}\frac1T\sum_t V_t(p)\right].
\]
Cauchy--Schwarz and minimization over the allocations give $\U_T\ge\mathcal H_T\ge\D_T$, and hence
\begin{equation}\label{eq:main-oracle-chain}
\sqrt{\frac{2L_F\Delta_F\U_T}{T}}
\ge\sqrt{\frac{2L_F\Delta_F\mathcal H_T}{T}}
\ge\sqrt{\frac{2L_F\Delta_F\D_T}{T}}.
\end{equation}
In particular, the gap between the uniform and optimal horizon-averaged second moments is
\begin{equation}\label{eq:diversity-gap}
\U_T-\D_T
=\frac{d+2}{T}\sum_{t=1}^T\E\left[\frac1N\sum_{i=1}^N\left(a_{i,t}-\frac1N\sum_{j=1}^N a_{j,t}\right)^2\right]\ge0.
\end{equation}
Equality holds if and only if all examples have the same gradient norm at every step, almost surely. This comparison uses a common parameter trajectory and does not rank the separate trajectories generated by different samplers.

\begin{table}[t]
\centering\small
\caption{Leading convergence terms on a common trajectory. The SAZO bound additionally includes the adaptation and finite-difference terms in Theorem~\ref{cor:rate}; hindsight allocation is a comparator.}\label{tab:oracle-comparison}
\begin{tabular}{@{}lll@{}}
\toprule
Sampling allocation & Average second moment & Leading term\\
\midrule
Uniform (MeZO) & $\U_T$ & $\sqrt{2L_F\Delta_F\U_T/T}$\\[3pt]
Best fixed in hindsight & $\mathcal H_T$ & $\sqrt{2L_F\Delta_F\mathcal H_T/T}$\\[3pt]
Optimal at each step (SAZO leading term) & $\D_T$ & $\sqrt{2L_F\Delta_F\D_T/T}$\\
\bottomrule
\end{tabular}
\end{table}

%% file: sections/experiments.tex
\section{Experiments}\label{sec:experiments}
We evaluate SAZO from five complementary perspectives: downstream performance on masked and generative language models, convergence across model scales, scaling with the training-population size, computational efficiency, and robustness to its sampling hyperparameters. Appendix~\ref{app:experimental-settings} provides the detailed experimental settings and additional~diagnostics.

\subsection{Performance and convergence across model families}\label{sec:performance-convergence}
\subsubsection{Masked language models}\label{sec:benchmark}
We first examine whether SAZO's optimization strategy translates into downstream accuracy. Using RoBERTa-large (350M) \citep{liu2019roberta} with $k=512$ examples per class, we consider six tasks spanning sentiment classification, natural language inference, and question classification. Table~\ref{tab:main} compares zero-shot learning, LP, and full-parameter and parameter-efficient FT, MeZO, HiZOO, and SAZO; across tasks, SAZO outperforms zero-shot and LP and approaches~FT.

\textbf{SAZO improves both full-parameter and parameter-efficient ZO tuning.} Relative to the stronger of MeZO and HiZOO, SAZO raises the average accuracy by 1.6 percentage points under full-parameter tuning and by 1.1 points under PEFT. The same pattern appears on individual tasks: the gains on TREC are 1.5 and 1.6 points for full-parameter tuning and PEFT, respectively, while those on MNLI are 1.3 and 1.8 points.

\begin{table}[t]
\centering\small
\setlength{\tabcolsep}{3pt}
\renewcommand{\arraystretch}{1.0}
\caption{Experiments on RoBERTa-large (350M parameters) with $k=512$ examples per class. LP: linear probing; FT: fine-tuning with Adam; PEFT: the better result between LoRA and prefix tuning for each task. MeZO and HiZOO results are reproduced. All tuning results report mean test accuracy, with standard deviations in subscripts. Bold: best ZO result.}\label{tab:main}
% Reproduced under the $k=512$ setup of \citet{malladi2023mezo}. Each FT, MeZO,
% and SAZO PEFT task entry is the larger of its LoRA and Prefix means. SAZO uses
% Experiment Results/Roberta k=512.md (three seeds).
% HiZOO: user-supplied results on 2026-09-24.
\begin{tabular*}{\linewidth}{@{\extracolsep{\fill}}lccccccc@{}}
\toprule
\multirow{2}{*}{Method} & \multicolumn{2}{c}{Sentiment} & \multicolumn{3}{c}{Natural language inference} & Question type & \multirow{2}{*}{Avg.}\\
\cmidrule(lr){2-3}\cmidrule(lr){4-6}\cmidrule(lr){7-7}
 & SST-2 & SST-5 & SNLI & MNLI & RTE & TREC & \\
\midrule
Zero-shot & $79.0$ & $35.5$ & $50.2$ & $48.8$ & $51.4$ & $32.0$ & $49.5$\\
LP & $91.3_{\pm0.5}$ & $51.7_{\pm0.5}$ & $80.9_{\pm1.0}$ & $71.5_{\pm1.1}$ & $73.1_{\pm1.5}$ & $89.4_{\pm0.5}$ & $76.3$\\
FT & $93.9_{\pm0.7}$ & $55.9_{\pm0.9}$ & $88.7_{\pm0.8}$ & $84.4_{\pm0.8}$ & $82.7_{\pm1.4}$ & $97.3_{\pm0.2}$ & $83.8$\\
FT (PEFT) & $94.2_{\pm0.2}$ & $55.3_{\pm0.7}$ & $88.3_{\pm0.5}$ & $83.9_{\pm0.6}$ & $83.2_{\pm1.3}$ & $97.4_{\pm0.2}$ & $83.7$\\
\midrule
MeZO & $93.3_{\pm0.7}$ & $53.2_{\pm1.4}$ & $83.0_{\pm1.0}$ & $78.3_{\pm0.5}$ & $78.6_{\pm2.0}$ & $94.3_{\pm1.3}$ & $80.1$\\
MeZO (PEFT) & $93.4_{\pm0.4}$ & $53.6_{\pm0.5}$ & $84.8_{\pm1.1}$ & $79.8_{\pm1.2}$ & $77.6_{\pm1.3}$ & $95.0_{\pm0.7}$ & $80.7$\\
\midrule
HiZOO & $93.5_{\pm0.4}$ & $53.5_{\pm1.2}$ & $83.3_{\pm1.4}$ & $77.2_{\pm1.5}$ & $79.1_{\pm1.2}$ & $94.9_{\pm1.7}$ & $80.3$\\
HiZOO (PEFT) & $\mathbf{94.3}_{\pm0.5}$ & $54.2_{\pm0.4}$ & $85.7_{\pm0.7}$ & $79.7_{\pm1.0}$ & $80.9_{\pm1.6}$ & $93.1_{\pm0.5}$ & $81.3$\\
\midrule
SAZO & $\mathbf{94.3}_{\pm0.5}$ & $53.6_{\pm0.8}$ & $85.3_{\pm0.4}$ & $79.6_{\pm0.3}$ & $\mathbf{82.3}_{\pm2.4}$ & $96.4_{\pm0.6}$ & $81.9$\\
SAZO (PEFT) & $94.0_{\pm0.8}$ & $\mathbf{54.5}_{\pm0.5}$ & $\mathbf{86.5}_{\pm0.4}$ & $\mathbf{81.6}_{\pm1.4}$ & $80.9_{\pm1.7}$ & $\mathbf{96.6}_{\pm0.5}$ & $\mathbf{82.4}$\\
\bottomrule
\end{tabular*}
\end{table}

\subsubsection{Generative language models}\label{sec:autoregressive}
We then test whether the benefits extend to generative language models (GLMs). Table~\ref{tab:opt13b} compares full-parameter SAZO, MeZO, and HiZOO on OPT-1.3B \citep{zhang2022opt}, Phi-2 (2.7B), and Llama3-8B across eight classification, multiple-choice, and question-answering tasks.

\textbf{SAZO achieves SOTA average performance compared with other ZO optimizers.} On OPT-1.3B, Phi-2, and Llama3-8B, SAZO improves the average over the stronger of MeZO and HiZOO by 1.2, 1.7, and 3.9 percentage points, respectively. The gains are especially large on Llama3-8B CB and DROP, where SAZO exceeds the stronger baseline by 10.8 and 8.4 points.

\begin{table}[!htb]
\centering\footnotesize
\setlength{\tabcolsep}{2pt}
\renewcommand{\arraystretch}{1.0}
\caption{Full-parameter generative-LM results. MeZO and HiZOO results are reproduced. }\label{tab:opt13b}
\begin{tabular*}{\linewidth}{@{\extracolsep{\fill}}llccccccccc@{}}
\toprule
\multirow{2}{*}{Model} & \multirow{2}{*}{Method} & \multicolumn{4}{c}{Classification} & \multicolumn{2}{c}{Multiple choice} & \multicolumn{2}{c}{Question answering} & \multirow{2}{*}{Avg.}\\
\cmidrule(lr){3-6}\cmidrule(lr){7-8}\cmidrule(lr){9-10}
 & & SST-2 & BoolQ & CB & WiC & COPA & MultiRC & SQuAD & DROP & \\
\midrule
\multirow{3}{*}{OPT-1.3B}
 & MeZO & 88.3 & 64.0 & 69.6 & 57.4 & 76.0 & \textbf{61.9} & 73.1 & \textbf{23.7} & 64.2\\
 & HiZOO & 90.7 & 64.0 & \textbf{71.4} & 58.2 & 75.0 & 60.1 & \textbf{76.6} & 21.3 & 64.7\\
 & SAZO & \textbf{90.8} & \textbf{65.0} & \textbf{71.4} & \textbf{60.8} & \textbf{77.0} & \textbf{61.9} & \textbf{76.6} & 23.3 & \textbf{65.9}\\
\midrule
\multirow{3}{*}{Phi-2}
 & MeZO & 86.6 & 72.4 & 75.0 & 54.4 & 86.0 & 78.2 & 85.7 & 37.8 & 72.0\\
 & HiZOO & \textbf{88.9} & 72.0 & 75.2 & \textbf{59.2} & 86.0 & \textbf{79.2} & 85.7 & 36.2 & 72.8\\
 & SAZO & 88.5 & \textbf{78.2} & \textbf{76.8} & 57.2 & \textbf{89.0} & \textbf{79.2} & \textbf{86.0} & \textbf{41.1} & \textbf{74.5}\\
\midrule
\multirow{3}{*}{Llama3-8B}
 & MeZO & 92.2 & 76.7 & 69.6 & \textbf{57.8} & 88.0 & 77.6 & 86.7 & 57.1 & 75.7\\
 & HiZOO & 94.3 & 77.1 & 69.6 & 57.7 & 89.0 & 77.9 & 87.7 & 49.4 & 75.3\\
 & SAZO & \textbf{94.4} & \textbf{80.8} & \textbf{80.4} & 56.6 & \textbf{91.0} & \textbf{78.8} & \textbf{89.3} & \textbf{65.5} & \textbf{79.6}\\
\bottomrule
\end{tabular*}
\end{table}

\subsubsection{Convergence across model scales}
\textbf{SAZO accelerates optimization across model scales.} Figure~\ref{fig:intro-loss} compares SST-2 training-loss trajectories on RoBERTa-large, Phi-2, and Llama 3 at a common forward-pass budget of 24K. Each method uses its optimal learning rate from the same search grid for each model; detailed settings are provided in Appendices~\ref{app:roberta-intro} and~\ref{app:extended-models}. SAZO reaches MeZO's endpoint losses in 4K, 6K, and 6K passes, while HiZOO requires 10.5K, 9K, and 6K passes, respectively.

\begin{figure}[!htb]
\centering
\includegraphics[width=\linewidth]{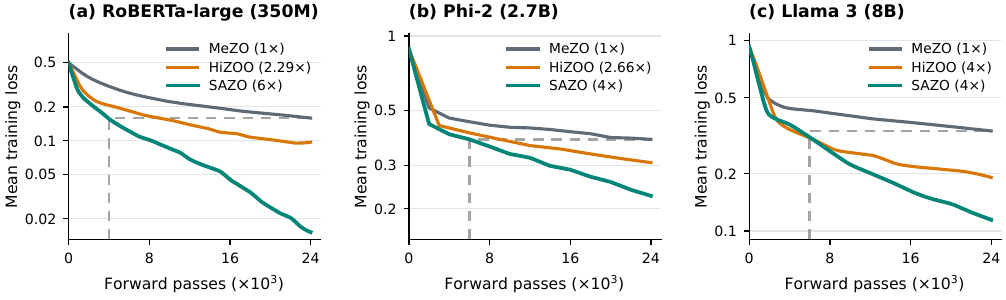}
\caption{\textbf{SAZO substantially accelerates convergence across model scales.} At 24K forward passes, SAZO reaches MeZO's endpoint loss $6\times$, $4\times$, and $4\times$ faster in (a)--(c), respectively; HiZOO achieves $2.29\times$, $2.66\times$, and $4\times$.}\label{fig:intro-loss}
\end{figure}

\subsection{Scaling with training-population size}\label{sec:population-scaling}
\textbf{SAZO remains effective as the training population grows.} Having established gains across model scales, we vary the amount of training data. Figure~\ref{fig:population-scaling} compares SST-2 test accuracy as the RoBERTa-large training population increases from $N=1{,}024$ to $4N$. SAZO outperforms MeZO at every scale, with gains of 0.6, 1.1, 1.4, and 1.5 percentage points at $N$, $2N$, $3N$, and $4N$, respectively.

\begin{figure}[!htbp]
\centering
\includegraphics[width=.60\linewidth]{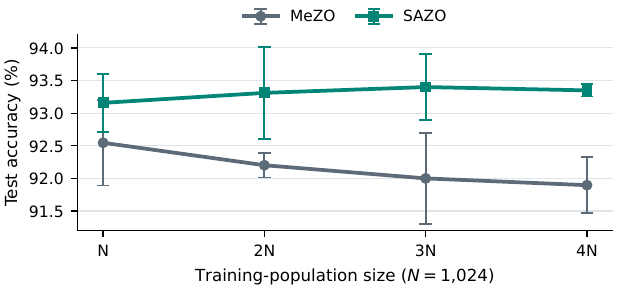}
\caption{RoBERTa-large SST-2 test accuracy as the training population scales from $N=1{,}024$ to $4N$. Curves show mean $\pm$ standard deviation over three seeds. See Appendix~\ref{app:population-scaling} for details.}\label{fig:population-scaling}
\end{figure}
\FloatBarrier

\subsection{Memory and wall-clock efficiency}\label{sec:efficiency}
We finally examine whether the gains above preserve the computational advantages of ZO optimization. \textbf{SAZO retains MeZO's memory footprint.} Across the OPT configurations in Figure~\ref{fig:memory}, the measured GB values for SAZO match those of MeZO. Compared with first-order methods, SAZO reduces memory requirements by factors of $6.8$--$12.2\times$ for full FT and $3.3$--$6.1\times$ for prefix FT (26 versus 316 GB at 13B for full FT). Thus, the downstream improvements do not sacrifice the low-memory advantage that makes ZO post-training attractive.

\textbf{SAZO also maintains low per-update latency.} Table~\ref{tab:wallclock} shows that its per-update time remains comparable to MeZO across the evaluated models. Together with faster loss reduction, this yields more effective updates at MeZO-like cost.

\begin{figure}[!htbp]
\centering
\includegraphics[width=.92\linewidth]{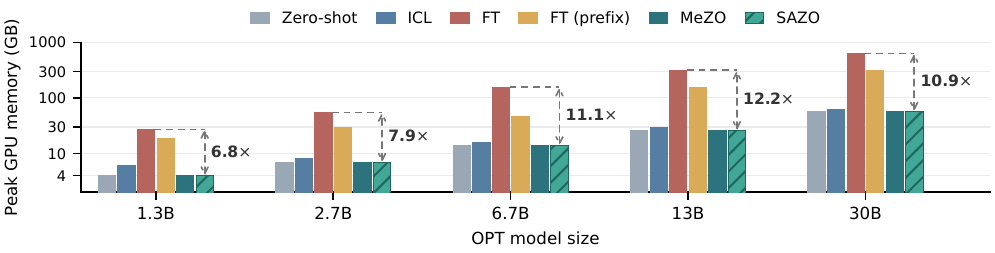}
\caption{Peak GPU memory on MultiRC (mean length 400 tokens; batch size 1/device), summed across GPUs. Arrows show FT/SAZO ratios.}\label{fig:memory}
\end{figure}

\begin{table}[!htbp]
\centering\small
\caption{Wall-clock time per update (seconds).}\label{tab:wallclock}
% Source: Experiment Results/wall clock.png; Ours is SAZO. All values are seconds/update.
\begin{tabular}{@{}lcccc@{}}
\toprule
Method & RoBERTa-large & OPT-1.3B & Phi-2 & Llama3-8B\\
\midrule
MeZO & 0.1153 & 0.1836 & 0.4235 & 0.9490\\
HiZOO & 0.2738 & 0.4793 & 1.0519 & 2.4798\\
SAZO & 0.1128 & 0.1839 & 0.4298 & 0.9805\\
\bottomrule
\end{tabular}
\end{table}

\subsection{Ablation studies}\label{sec:beta-sensitivity}
We conclude by examining sensitivity to the importance-score update. \textbf{SAZO is robust to the smoothing coefficient.} Figure~\ref{fig:hyperparameter-ablation} varies this coefficient on SST-2. Across $\beta\in\{0.005,0.05,0.5\}$, the trajectories remain stable and reach closely clustered endpoint losses, with $\beta=0.05$ performing best. In contrast, $\beta=1$ discards the accumulated history, relies only on the latest directional observation, and diverges. This contrast highlights the importance of per-example moving averages for tracking evolving importance from intermittent feedback while smoothing directional noise. SAZO is similarly robust to the exploration coefficient $\varepsilon$ over a broad range; Appendix~\ref{app:hyperparameter-ablation} provides the full sweep and experimental details.

\begin{figure}[!htbp]
\centering
\includegraphics[width=.72\linewidth]{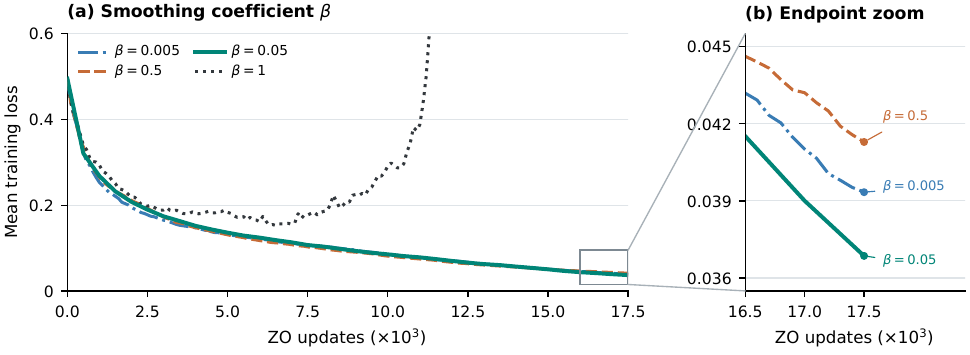}
\caption{Smoothing coefficient ablation on SST-2. Detailed settings are provided in Appendix~\ref{app:hyperparameter-ablation}.}\label{fig:hyperparameter-ablation}
\end{figure}
\FloatBarrier

%% file: sections/conclusion.tex
\section{Conclusion}
In this work, we propose SAZO, a sample-adaptive zeroth-order optimizer that reuses existing directional differences to estimate sample importance without additional model evaluations. Per-example exponential moving averages track how this importance evolves from sparse minibatch feedback and require only $O(N)$ scalar state. Our analysis asymptotically matches the variance-optimal dynamic oracle's leading convergence term, yielding a smaller convergence coefficient than uniform sampling when sample importance is heterogeneous. Extensive experiments demonstrate improved average downstream performance and up to $6\times$ faster convergence to the same training loss in terms of query counts.

%% file: sections/appendix.tex
\section{Proofs of the main results}\label{app:proofs}
Let $\mathcal F_t$ denote the history before the current draws, so $x_t,m_t,p_t$ are $\mathcal F_t$-measurable and $\E_t$ is conditional expectation. For the proofs, it is convenient to remove the common Gaussian factor from the second moments in Section~\ref{sec:theory}. Throughout the appendices, define
\[
\widetilde V_t(p)=\frac{V_t(p)}{d+2}=\frac1{N^2}\sum_i\frac{a_{i,t}^2}{p_i},\qquad
\bar a_t=\frac1N\sum_i a_{i,t},\qquad
\widetilde V_t^\star=\frac{V_t^\star}{d+2}=\bar a_t^2,
\]
Writing $\rho_t=V_t(p_t)-V_t^\star$, use $\widetilde\rho_t=\rho_t/(d+2)$, $\widetilde\Reg_T=\Reg_T/(d+2)$, $\widetilde\D_T=\D_T/(d+2)$, $\widetilde\U_T=\U_T/(d+2)$, $\widetilde{\mathcal H}_T=\mathcal H_T/(d+2)$, and $\widetilde K_T=K_T/(d+2)$. Also set $C_d=L_F(d+2)$. These are exact rescalings of the main-text~quantities.

For the explicit radius calculations below, set $D_6=d(d+2)(d+4)$, $D_8=d(d+2)(d+4)(d+6)$, $H_2^2=N^{-1}\sum_iH_i^2$, $J_{i,t}=\pi H_i^2D_6\mu_t^4/72$, and $W_t=H_2^2D_8\mu_t^4/36$. Write $\overline W_T=T^{-1}\sum_tW_t$ and $\mathcal B_T=T^{-1}\sum_{t,i}\E[\beta^2p_{i,t}J_{i,t}]$. With $r=\beta\varepsilon/N$, define
\[
R_{\mu,T}=\frac{\mathcal B_T+NL_2^2\eta^2\overline W_T/\varepsilon}{N\varepsilon r^2},\qquad
\Omega_{\mu,T}=\sqrt{\widetilde{\D}_T\overline W_T}
+\frac{L_F\eta}{\varepsilon}\sqrt{\frac{(d+2)A_T\overline W_T}{N}}
+\frac{L_F\eta}{\varepsilon}\overline W_T.
\]
These quantities expand the compact radius errors used in the main text. At a fixed horizon and fixed algorithm parameters, $R_{\mu,T}=O(\mu^4)$ and $\Omega_{\mu,T}=O(\mu^2)$ for $\mu=\max_t\mu_t\le1$, with moment dependencies retained here. For local observation variance statements, the hidden coefficient also contains $a_{i,t}$ through $a_{i,t}\sqrt{J_{i,t}}$.

\input{sections/theory_details}

\subsection{Gaussian identities and the sampling oracle}
In this appendix only, write $s_{i,t}^{\mu_t}=s_{i,t}$ for the actual difference used in the main text and introduce $s_{i,t}^0=z_t^\top g_{i,t}$ as a Taylor expansion term at the same current parameter vector. There is no auxiliary optimization trajectory. Condition on $\mathcal F_t$ and suppress $t$. For any fixed $g$, $z^\top g$ has distribution $\norm g Z$ with $Z\sim\mathcal N(0,1)$. Since $\E|Z|=\sqrt{2/\pi}$ and $\E Z^2=1$, the Gaussian moments of $Y_i^0=\sqrt{\pi/2}|s_i^0|$ follow. Adding the Taylor remainder bounded in Appendix~\ref{app:fd} gives Eq.~\eqref{eq:obs-moments}. To compute the update moment, rotate $g$ onto the first coordinate:
\[
\E_z\norm{zz^\top g}^2
=\norm g^2\E\left[z_1^2\sum_{j=1}^dz_j^2\right]
=(3+d-1)\norm g^2.
\]
Also $\E zz^\top=I_d$. Averaging over the selected index gives the leading moment $(d+2)\widetilde V_t(p_t)$. Adding the Taylor remainder gives Eq.~\eqref{eq:update-moments}. Finally,
\[
\left(\sum_i a_i\right)^2
=\left(\sum_i\frac{a_i}{\sqrt{p_i}}\sqrt{p_i}\right)^2
\le\sum_i\frac{a_i^2}{p_i}.
\]
Equality holds for $p_i\propto a_i$ on the positive-norm support. Infima on the boundary are understood by limits, with zero-norm terms contributing zero.

\subsection{Stationary variance of the EMA}\label{app:ema-variance}
Fix the parameter vector and index $i$, and enumerate its visits by $k$. At a fixed radius, with fresh independent Gaussian directions, let $a_i^\mu=\E Y_i^{(k)}$ and $v_i^\mu=\operatorname{Var}(Y_i^{(k)})$. The central-difference estimates obey $a_i^\mu=a_i+O(\mu^2)$ with the explicit coefficients in Proposition~\ref{prop:moments}. For deterministic initial score $m_i^{(0)}$, the recursion $m_i^{(k)}=(1-\beta)m_i^{(k-1)}+\beta Y_i^{(k)}$~gives
\[
\E(m_i^{(k)}-a_i^\mu)^2=(1-\beta)^{2k}(m_i^{(0)}-a_i^\mu)^2
+\frac{\beta}{2-\beta}\bigl[1-(1-\beta)^{2k}\bigr]v_i^\mu.
\]
This follows by centering the recursion and summing a geometric series. At visit times, the variance converges to $\beta v_i^\mu/(2-\beta)$. Averaging reduces variance but retains the finite-difference bias in $a_i^\mu$. The result concerns a fixed target; it does not remove the drift term on an evolving optimization trajectory.

\subsection{A global normalization bound}
\begin{lemma}\label{lem:normalization}
For $a\ge0$, $m>0$, $0<\varepsilon\le1/2$, and $p=(1-\varepsilon)m/(\sum_i m_i)+\varepsilon\mathbf1/N$, let $G=N^{-1}\sum_i a_i$ and $e=m-a$.~Then
\begin{equation}\label{eq:convert}
\widetilde V(p)-G^2\le\frac8{N\varepsilon}\norm e^2+4\varepsilon G^2.
\end{equation}
\end{lemma}
\begin{proof}
Let $S=\sum_i a_i$, $M=\sum_i m_i$, $q=m/M$, and $u=\mathbf1/N$. To retain the normalization by $N^2$, work directly with the weighted squared norm $\norm v_{p^{-1}}^2=\sum_i v_i^2/p_i$. Expanding the square gives
\[
N^2\bigl[\widetilde V(p)-G^2\bigr]=\norm{a-Sp}_{p^{-1}}^2.
\]
Write $\delta=a-Sq=-e+(M-S)q$. Since $p_i\ge\varepsilon/N$, $p_i\ge(1-\varepsilon)q_i$, and $(M-S)^2\le N\norm e^2$,
\begin{align*}
\norm\delta_{p^{-1}}^2
&\le2\sum_i\frac{e_i^2}{p_i}+2(M-S)^2\sum_i\frac{q_i^2}{p_i}\\
&\le\left(\frac{2N}{\varepsilon}+\frac{2N}{1-\varepsilon}\right)\norm e^2
=\frac{2N}{\varepsilon(1-\varepsilon)}\norm e^2.
\end{align*}
The exploration contribution also admits a weighted bound:
\[
\varepsilon^2\norm{q-u}_{p^{-1}}^2
=\norm{q-p}_{p^{-1}}^2
=\sum_i\frac{q_i^2}{p_i}-1
\le\frac{\varepsilon}{1-\varepsilon}.
\]
Using $a-Sp=\delta+S\varepsilon(q-u)$ therefore yields
\begin{align*}
\widetilde V(p)-G^2
&\le\frac{4}{N\varepsilon(1-\varepsilon)}\norm e^2
+\frac{2\varepsilon}{1-\varepsilon}G^2\\
&\le\frac8{N\varepsilon}\norm e^2+4\varepsilon G^2.
\end{align*}
The argument includes $S=0$ and requires no closeness condition between $M$ and $S$.
\end{proof}

\subsection{Sparse EMA with correlated target drift}
The following lemma allows dependent inclusions and shared directional observations. Given the current history, let $O_i$ be an inclusion indicator independent of the fresh Gaussian direction, with $\pi_i=\E_t O_i$. The indicators for different coordinates may be dependent. For this local lemma, let $Y_i=Y_i^0$ denote the leading Gaussian observation and update $m_i^+=m_i+\beta O_i(Y_i-m_i)$, and suppose $\beta\pi_i\ge r>0$ for all steps and coordinates. Define $\delta a_{i,t}=a_{i,t+1}-a_{i,t}$.

\begin{lemma}[Visited-coordinate tracking]\label{lem:sparse}
With $E_0=\E\norm{m_1-a_1}^2$ and $Q_T=\sum_{t<T}\E\norm{a_{t+1}-a_t}^2$,
\begin{equation}\label{eq:general-mse}
\sum_t\E\norm{m_t-a_t}^2
\le\frac{2E_0}{r}
+\frac{3\sigma_G^2\beta^2}{r}\sum_{t,i}\E[\pi_{i,t}a_{i,t}^2]
+\frac{12Q_T}{r^2}.
\end{equation}
\end{lemma}
\begin{proof}
For $e_i=m_i-a_i$ and $\nu_i=Y_i-a_i$, write
\[
X_i=(1-\beta O_i)e_i+\beta O_i\nu_i,\qquad
e_{i,t+1}=X_i-\delta a_{i,t}.
\]
Conditional centering and independence of $O_i$ and the new direction imply
\[
\E_t X_i^2=[1-(2\beta-\beta^2)\pi_i]e_i^2
+\beta^2\pi_i\sigma_G^2a_i^2
\le(1-r)e_i^2+\beta^2\pi_i\sigma_G^2a_i^2.
\]
Young's inequality with parameter $r/2$~gives
\[
\E_t e_{i,t+1}^2
\le(1-r/2)e_i^2+\tfrac32\beta^2\pi_i\sigma_G^2a_i^2
+\frac3r\E_t(\delta a_{i,t})^2,
\]
because $(1+r/2)(1-r)\le1-r/2$ and $1+2/r\le3/r$. Summing for $t<T$ and over $i$, including the terminal contribution using $r/2\le1$, gives Eq.~\eqref{eq:general-mse}, in fact with 6 instead of 12 for the drift constant. We retain 12 to use common constants with the finite-difference extension. The argument never assumes that $X_i$ and $\delta a_{i,t}$ are independent.
\end{proof}

For Algorithm~\ref{alg:single}, $\pi_i=p_i$ and $r=\beta\varepsilon/N$. Use $T^{-1}\sum_{t,i}\E[p_{i,t}a_{i,t}^2]\le A_T$. Adding the finite-difference observation remainder as in Appendix~\ref{app:fd} gives, with $e_t=m_t-a_t$,
\begin{equation}\label{eq:mse-outline}
\sum_t\E\norm{e_t}^2
\le\frac{2E_0}{r}+\frac{3\sigma_G^2\beta^2}{r}T A_T+\frac{12Q_T}{r^2}+O(T\mu^4).
\end{equation}

\subsection{Closing the drift and proving stationarity}
All quantities are evaluated on the actual finite-difference trajectory. The perturbed sparse-EMA inequality and normalization give
\[
\widetilde{\Reg}_T\le\frac{16E_0}{N\varepsilon r}
+\frac{24\sigma_G^2\beta^2}{N\varepsilon r}T A_T
+\frac{96}{N\varepsilon r^2}(Q_T+T\mathcal B_T)+4\varepsilon T\widetilde{\D}_T.
\]
Gradient continuity and the second-moment bound imply
\[
Q_T\le NL_2^2\sum_{t<T}\E\norm{x_{t+1}-x_t}^2\le2N\Lambda_d\eta^2(T\widetilde{\D}_T+\widetilde{\Reg}_T)
+2NL_2^2\eta^2T\overline W_T/\varepsilon.
\]
Substitution and absorption prove Theorem~\ref{thm:tracking} after multiplication by $d+2$, with additive remainder at most $192(d+2)R_{\mu,T}/(1-\kappa)$.

For descent, write $\widehat g_t=\widehat g_t^0+h_t$ at the same parameters, selected index, and direction, and let $M_t=(d+2)\widetilde V_t(p_t)$. The Taylor bounds give
\[
\norm{\E_t h_t}^2\le W_t,\qquad
\E_t\norm{\widehat g_t}^2\le M_t+2\sqrt{M_tW_t/\varepsilon}+W_t/\varepsilon.
\]
Smoothness implies
\[
\E_t F(x_{t+1})\le F(x_t)-\eta\norm{\nabla F(x_t)}^2
+\eta\norm{\nabla F(x_t)}\sqrt{W_t}
+\frac{L_F\eta^2}{2}\left(M_t+2\sqrt{M_tW_t/\varepsilon}+W_t/\varepsilon\right).
\]
Using $\Psi_T\le\widetilde{\D}_T$ and $T^{-1}\sum_t\E M_t\le(d+2)A_T/(N\varepsilon)$, Cauchy--Schwarz over time and randomness bounds the additional terms by $\Omega_{\mu,T}$. Summation proves Lemma~\ref{thm:outer} without multiplying its leading term by the coarse moment bound's factor two.

For Theorem~\ref{cor:rate}, substitution of the schedules and $r=\beta\varepsilon/N$ gives, when $\kappa\le1/2$,
\[
\widetilde{\Reg}_T/T\le\frac{32E_0}{c_\varepsilon^2c_\beta T^{1/2}}
+\frac{384\widetilde K_T}{T^{1/10}}+384R_{\mu,T}.
\]
Insert this into the descent lemma. Holding $\widetilde{\D}_T$ fixed, the leading expression is minimized at $c_\eta=\sqrt{2\Delta_F/(C_d\widetilde{\D}_T)}$. Before collecting radius terms, the result is
\[
\Psi_T\le\sqrt{\frac{2C_d\Delta_F\widetilde{\D}_T}{T}}+192\left(\frac{C_dc_\eta \widetilde K_T}{T^{3/5}}+\frac{C_dc_\eta E_0}{c_\varepsilon^2c_\beta T}+\frac{C_dc_\eta R_{\mu,T}}{\sqrt T}+\Omega_{\mu,T}\right).
\]
This gives Eq.~\eqref{eq:rate-explicit} after rescaling, with the radius-error convention stated in the main text. The general descent lemma covers undefined or infeasible oracle coefficients; the choice is conditional on the resulting trajectory quantity.

\section{Parameter balance and comparator qualifications}\label{app:balance}
The radius-independent noise, drift, and exploration contributions have scales
\[
\frac{\beta A_T}{\varepsilon^2},\qquad
\frac{N^2\Lambda_d\eta^2\widetilde{\D}_T}{\varepsilon^3\beta^2},\qquad \varepsilon\widetilde{\D}_T.
\]
For positive moments and smoothness, their formal balance is
\begin{align}
\varepsilon&\asymp N^{1/5}\Lambda_d^{1/10}\eta^{1/5}(A_T/\widetilde{\D}_T)^{1/5},\label{eq:balance-eps}\\
\beta&\asymp N^{3/5}\Lambda_d^{3/10}\eta^{3/5}(\widetilde{\D}_T/A_T)^{2/5}.\label{eq:balance-beta}
\end{align}
If these choices and the resulting trajectory satisfy feasibility and $\kappa\le1/2$,~then
\begin{equation}\label{eq:balanced}
\widetilde{\Reg}_T/T\lesssim N^{1/5}\Lambda_d^{1/10}\eta^{1/5}A_T^{1/5}\widetilde{\D}_T^{4/5}
+\frac{E_0}{TN\sqrt{\Lambda_d}\eta}+O(R_{\mu,T}).
\end{equation}
Only universal constants are suppressed. These choices require $A_T,\widetilde{\D}_T$, which depend on the parameters and resulting trajectory; this is a conditional balance, not an online tuning rule or an existence theorem for self-consistent parameters. The initialization term follows from $\varepsilon^2\beta\asymp N\sqrt{\Lambda_d}\eta$. The finite-difference correction remains additive.

Balancing noise and drift alone gives $\beta\varepsilon^{1/3}/\eta^{2/3}\asymp(N^2\Lambda_d\widetilde{\D}_T/A_T)^{1/3}$. A common numerical value across tasks is an empirical tuning constraint; the right side is problem-dependent.

\subsection{Three fixed and dynamic costs}\label{app:comparators}
In addition to $\widetilde{\mathcal H}_T$ and $\widetilde{\U}_T$, define
\[
\widetilde{\mathcal G}_T=\inf_{p\in\Delta_N}\frac1T\sum_t\E \widetilde V_t(p)
=\left(\frac1N\sum_i\sqrt{\frac1T\sum_t\E a_{i,t}^2}\right)^2.
\]
For each realized trajectory,
\[
\inf_p\frac1T\sum_t \widetilde V_t(p)
=\left(\frac1N\sum_i\sqrt{\frac1T\sum_t a_{i,t}^2}\right)^2.
\]
Pointwise minimization, $\E\inf\le\inf\E$, and evaluation at the uniform distribution establish
\begin{equation}\label{eq:oracle-chain}
\widetilde{\D}_T\le\widetilde{\mathcal H}_T\le\widetilde{\mathcal G}_T\le\widetilde{\U}_T.
\end{equation}
$\widetilde{\mathcal H}_T$ is the expected hindsight cost, while $\widetilde{\mathcal G}_T$ chooses a fixed distribution before taking the trajectory expectation. Interchanging them would change the comparator. For $a_t=c_tv$ with fixed $v\ge0$, the dynamic and hindsight fixed costs coincide. For the deterministic two-step example $a_1=(1,0)$, $a_2=(0,1)$, $\widetilde{\D}_T=1/4$ while $\widetilde{\mathcal H}_T=\widetilde{\mathcal G}_T=\widetilde{\U}_T=1/2$. This example illustrates a cost hierarchy and is not asserted to be a realized training trajectory.

Equation~\eqref{eq:diversity-gap} follows by expanding the squared deviations around $\bar a_t$ and multiplying by $d+2$. Consequently, $\widetilde{\D}_T<\widetilde{\U}_T$ whenever the example norms differ at some step with positive probability. The balanced expressions in Table~\ref{tab:oracle-comparison} follow by minimizing $\Delta_F/c+C_dcX/2$ over $c>0$ for each fixed positive normalized comparator quantity $X$. This is an algebraic comparison on the same trajectory; it does not require treating a hindsight-selected distribution as an implementable~sampler.

\section{Finite-difference error}\label{app:fd}
This appendix retains the original, unsmoothed gradients as the reference quantities. It covers Algorithm~\ref{alg:single} and the with-replacement minibatch update in Appendix~\ref{app:batch}.

The Hessian regularity is included in Assumption~\ref{ass:smooth}; set $H_2^2=N^{-1}\sum_iH_i^2$.
Let $D_6=d(d+2)(d+4)$ and $D_8=d(d+2)(d+4)(d+6)$. Taylor's integral remainder gives
\[
b_{i,t}=s_{i,t}^{\mu_t}-z_t^\top g_{i,t},\qquad
|b_{i,t}|\le\frac{H_i\mu_t^2}{6}\norm{z_t}^3.
\]
For $\delta_{i,t}=\sqrt{\pi/2}(|s_{i,t}^{\mu_t}|-|z_t^\top g_{i,t}|)$, define
\[
J_{i,t}=\frac{\pi H_i^2D_6}{72}\mu_t^4,\qquad
W_t=\frac{H_2^2D_8}{36}\mu_t^4,\qquad
\overline W_T=T^{-1}\sum_tW_t.
\]
The reverse triangle inequality and Gaussian radial moments imply
\begin{equation}\label{eq:fd-basic}
\E_t\delta_{i,t}^2\le J_{i,t},\qquad
\E_t\norm{b_{i,t}z_t}^2\le H_i^2D_8\mu_t^4/36.
\end{equation}

In this appendix take $B=1$ for Algorithm~\ref{alg:single}, or $B>1$ for the experimental implementation. Let $\pi_{i,t}$ denote the corresponding inclusion probability and let $r>0$ lower-bound $\beta\pi_{i,t}$ uniformly. Define
\[
\chi_B=\frac1{B\varepsilon}+1-\frac1B,\quad
A_{\pi,T}=\frac1T\sum_{t,i}\E[\pi_{i,t}a_{i,t}^2],\quad
\mathcal B_T=\frac1T\sum_{t,i}\E[\beta^2\pi_{i,t}J_{i,t}].
\]
For the main algorithm, take $r=\beta\varepsilon/N$ and $\pi_i=p_i$. For the batch extension, $\pi_{i,t}=1-(1-p_{i,t})^B$ and $r=\beta[1-(1-\varepsilon/N)^B]$.

\begin{lemma}[Outer perturbation]\label{lem:fd-update}
Let $h_t^\mu=\widehat g_t^\mu-\widehat g_t^0$, where the two estimators use the same current parameters, indices, and direction, with finite or ideal differences.~Then
\[
\norm{\E_t h_t^\mu}^2\le W_t,\qquad
\E_t\norm{h_t^\mu}^2\le\chi_B W_t.
\]
\end{lemma}
\begin{proof}
Conditional on $z$, the with-replacement sample average has expectation $N^{-1}\sum_i b_i z$. Jensen's inequality and $(N^{-1}\sum_i H_i)^2\le H_2^2$ prove the first bound. For $B$ independent with-replacement indices, the conditional second moment consists of $B^{-1}N^{-2}\sum_i\norm{b_i z}^2/p_i$ plus $(1-1/B)\norm{N^{-1}\sum_i b_i z}^2$. Use $p_i\ge\varepsilon/N$ and Eq.~\eqref{eq:fd-basic}.
\end{proof}

\begin{theorem}[Explicit finite-radius bound]\label{thm:finite}
Under Assumption~\ref{ass:smooth}, use the with-replacement update in Eq.~\eqref{eq:batch-update} and the once-per-included-example EMA in Eq.~\eqref{eq:batch-ema}, with $0<\varepsilon\le1/2$ and $0<\beta<1$. Let
\[
\kappa_\mu=\frac{192\Lambda_d\eta^2}{B\varepsilon r^2}<1.
\]
Along the actual finite-difference trajectory, define $\widetilde{\Reg}_T,\widetilde{\D}_T$ using the original gradient norms using the normalization above and Eq.~\eqref{eq:complexities}.~Then
\begin{equation}\label{eq:finite-regret}
\begin{split}
\frac{\widetilde{\Reg}_T}{T}\le\frac1{1-\kappa_\mu}\bigg[
&\frac{16E_0}{TN\varepsilon r}
+\frac{24\sigma_G^2\beta^2}{N\varepsilon r}A_{\pi,T}
+\frac{96\mathcal B_T}{N\varepsilon r^2}\\
&+\left(\frac{192\Lambda_d\eta^2}{\varepsilon r^2}+4\varepsilon\right)\widetilde{\D}_T
+\frac{192L_2^2\eta^2\chi_B}{\varepsilon r^2}\overline W_T\bigg].
\end{split}
\end{equation}
If $C_d\eta(1-1/B)\le1/4$,~then
\begin{equation}\label{eq:finite-outer}
\Psi_T\le\frac{4\Delta_F}{\eta T}
+\frac{4C_d\eta}{B}(\widetilde{\D}_T+\widetilde{\Reg}_T/T)
+(2+4L_F\eta\chi_B)\overline W_T.
\end{equation}
\end{theorem}
\begin{proof}
In Lemma~\ref{lem:sparse}, the new error is $X_i+\beta O_i\delta_i-\delta a_i$. Apply Young's inequality with parameter $r/2$, then bound the square of the last two terms by twice the sum of their squares. This gives
\[
\sum_t\E\norm{m_t-a_t}^2
\le\frac{2E_0}{r}+\frac{3\sigma_G^2\beta^2}{r}T A_{\pi,T}
+\frac{12}{r^2}(T\mathcal B_T+Q_T).
\]
No centering of $\delta_i$ is assumed. By the ideal batch second moment, Lemma~\ref{lem:fd-update}, and the two-term square inequality,
\[
\E_t\norm{\widehat g_t^\mu}^2
\le2(d+2)\left[\widetilde V_t(p_t)/B+(1-1/B)\norm{\nabla F(x_t)}^2\right]+2\chi_BW_t.
\]
\[
\text{Consequently,}\quad
Q_T\le2N\Lambda_d\eta^2(T\widetilde{\D}_T+\widetilde{\Reg}_T/B)
+2NL_2^2\eta^2\chi_BT\overline W_T.
\]
Substitute into Lemma~\ref{lem:normalization} and absorb the regret coefficient to prove Eq.~\eqref{eq:finite-regret}. For descent, use $\E_t\widehat g_t^\mu=\nabla F(x_t)+b_t$ with $\norm{b_t}^2\le W_t$, and $-\ip{\nabla F}{b_t}\le\norm{\nabla F}^2/2+W_t/2$. The coefficient of $\Psi_T$ is at least $1/2-C_d\eta(1-1/B)\ge1/4$, giving Eq.~\eqref{eq:finite-outer}.
\end{proof}

\paragraph{Radius dependence in the main theorem.}
For $M_{4,T}=T^{-1}\sum_t\mu_t^4$, the explicit bounds are
\[
\mathcal B_T\le\frac{\pi N\beta^2H_2^2D_6}{72}M_{4,T},\qquad
\overline W_T=\frac{H_2^2D_8}{36}M_{4,T}.
\]
For the main algorithm, with fixed problem and schedule coefficients, bounded trajectory moments, and the horizon schedules in Theorem~\ref{cor:rate}, one has $r=\Theta(T^{-2/5})$, $\chi_B=O(T^{1/10})$, and
\[
R_{\mu,T}=O(T^{3/10}M_{4,T}+M_{4,T}),\qquad
\Omega_{\mu,T}=O(\sqrt{M_{4,T}}+T^{-2/5}M_{4,T}).
\]
The constants in these two order summaries depend on the fixed coefficients explicitly displayed above and in the main definitions. The radius-independent remainder in the main theorem hides only universal constants. Its separate $O(\mu^2)$ hides the radius coefficients displayed here. Under the stated fixed-parameter and bounded-moment conditions, $\eta R_{\mu,T}=O(T^{-1/5}\mu^4+T^{-1/2}\mu^4)$ and $\Omega_{\mu,T}=O(\mu^2)$ uniformly in feasible $T$; thus the collected $O(\mu^2)$ is uniform in those horizons. Choosing $\mu_t=\mu_0T^{-\alpha}$ for every step of a run gives $M_{4,T}=\mu_0^4T^{-4\alpha}$. For $\alpha>3/10$, both $\Omega_{\mu,T}$ and $\eta R_{\mu,T}$ are $o(T^{-3/5})$. This condition applies to a horizon-dependent radius and does not assert the same conclusion for $\mu_t=\mu_0t^{-\alpha}$ with a fixed initial radius. The earlier coarse outer bound remains useful if one prefers squared-radius errors at the price of larger leading constants. Evaluation randomness and numerical rounding are not represented by these Taylor bounds.

\section{With-replacement minibatches and shared directions}\label{app:batch}
All experiments use with-replacement minibatch sampling. This appendix specifies the implementation and derives the effect of sharing one Gaussian direction within a batch. Algorithm~\ref{alg:single} and the main convergence results are its $B=1$~case.

\subsection{Minibatch update and reused evaluations}
At step $t$, form $p_t$ from the stored scores and draw $B$ indices independently with replacement, $I_{t,1},\ldots,I_{t,B}\sim p_t$, independently of $z_t\sim\mathcal N(0,I_d)$. Let
\[
 K_{i,t}=\sum_{b=1}^B\mathbf1\{I_{t,b}=i\},
 \qquad S_t=\{i:K_{i,t}>0\}.
\]
For each $i\in S_t$, evaluate the two-sided difference only once, using the common perturbations $x_t\pm\mu_tz_t$. If an example is drawn repeatedly, its stored difference is reused: it does not trigger another forward evaluation or a replacement draw. The model update is
\begin{equation}\label{eq:batch-update}
\widehat g_t
=\frac{z_t}{B}\sum_{b=1}^B\frac{s_{I_{t,b},t}}{Np_{I_{t,b},t}}
=\frac{z_t}{BN}\sum_{i\in S_t}\frac{K_{i,t}s_{i,t}}{p_{i,t}},
\qquad x_{t+1}=x_t-\eta\widehat g_t.
\end{equation}
Thus evaluation reuse preserves the multiplicities of the original $B$ draws. Each included example updates its score~once:
\begin{equation}\label{eq:batch-ema}
 m_{i,t+1}=
 \begin{cases}
 (1-\beta)m_{i,t}+\beta\sqrt{\dfrac\pi2}|s_{i,t}|,&i\in S_t,\\[3pt]
 m_{i,t},&i\notin S_t.
 \end{cases}
\end{equation}
The distribution remains fixed within a batch; the new scores determine $p_{t+1}$. No inverse-probability factor is applied to the EMA. The two perturbed batches contain $|S_t|$ examples each, requiring $2|S_t|\le2B$ example-level function evaluations. The expected number of distinct examples is $\sum_i[1-(1-p_{i,t})^B]$.

\subsection{Effect of the shared direction on variance}
For fixed $x_t,p_t,z_t$, let $\bar s_t=N^{-1}\sum_i s_{i,t}$. Independence of the index draws gives the exact finite-difference identities
\begin{align}
 \E_{I\mid z_t}\widehat g_t&=\bar s_tz_t,\label{eq:batch-conditional-mean}\\
 \E_{I\mid z_t}\norm{\widehat g_t}^2
 &=\norm{z_t}^2\left[
 \frac{1}{BN^2}\sum_i\frac{s_{i,t}^2}{p_{i,t}}
 +\left(1-\frac1B\right)\bar s_t^2\right].\label{eq:batch-conditional-second}
\end{align}
Hence sharing a direction and reusing duplicate evaluations preserve the expected full-data finite-difference update. The differences for different examples need not be independent.

Write $\bar g_t=\nabla F(x_t)$. Under Assumption~\ref{ass:smooth}, integrating over the shared Gaussian direction yields
\begin{proposition}[Shared-direction minibatch moments]\label{prop:batch-moments}
For the update in Eq.~\eqref{eq:batch-update},
\begin{equation}\label{eq:batch-second}
\begin{gathered}
 \E_t\widehat g_t=\bar g_t+O(\mu^2),\\[4pt]
 \E_t\norm{\widehat g_t}^2
 =(d+2)\left[\frac{\widetilde V_t(p_t)}B+
 \left(1-\frac1B\right)\norm{\bar g_t}^2\right]+O(\mu^2).
\end{gathered}
\end{equation}
Its total variance is
\begin{equation}\label{eq:batch-variance}
 \E_t\norm{\widehat g_t-\E_t\widehat g_t}^2
 =(d+1)\norm{\bar g_t}^2
 +\frac{d+2}{B}\left[\widetilde V_t(p_t)-\norm{\bar g_t}^2\right]+O(\mu^2).
\end{equation}
The radius terms follow the convention in Appendix~\ref{app:fd}; the displayed leading terms are exact for directional derivatives.
\end{proposition}
\begin{proof}
The central-difference remainder satisfies $s_{i,t}=z_t^\top g_{i,t}+b_{i,t}$ with $|b_{i,t}|\le H_i\mu_t^2\norm{z_t}^3/6$. Applying this expansion directly to Eqs.~\eqref{eq:batch-conditional-mean}--\eqref{eq:batch-conditional-second}, the Gaussian identities
\[
 \E[zz^\top]=I_d,
 \qquad \E[\norm{z}^2(z^\top v)^2]=(d+2)\norm{v}^2
\]
give the displayed leading terms. Cross terms involving $b_{i,t}$ are $O(\mu^2)$ by Cauchy--Schwarz and the finite-radius moment bounds; squared remainders are $O(\mu^4)$. Subtracting $\norm{\E_t\widehat g_t}^2=\norm{\bar g_t}^2+O(\mu^2)$ gives Eq.~\eqref{eq:batch-variance}.
\end{proof}

\paragraph{Advantage over MeZO under the same shared direction.}
MeZO also shares one Gaussian direction across a minibatch. Its leading directional variance is therefore the same $(d+1)\norm{\bar g_t}^2$ as SAZO's; the allocation controls the remaining term. At the same parameter vector and batch size, write $p^{\mathrm{unif}}=\mathbf1/N$. The variance reduction relative to MeZO is
\begin{equation}\label{eq:batch-mezo-gap}
 \operatorname{Var}_t(\widehat g_t^{\mathrm{MeZO}})
 -\operatorname{Var}_t(\widehat g_t^{\mathrm{SAZO}})
 =\frac{d+2}{B}\left[\widetilde V_t(p^{\mathrm{unif}})-\widetilde V_t(p_t)\right]+O(\mu^2),
\end{equation}
where $\operatorname{Var}_t(X)=\E_t\norm{X-\E_tX}^2$. For the norm-proportional oracle, the leading gain is
\begin{equation}\label{eq:batch-oracle-gain}
 \frac{d+2}{B}\left[\widetilde V_t(p^{\mathrm{unif}})-\bar a_t^2\right]
 =\frac{d+2}{BN}\sum_i(a_{i,t}-\bar a_t)^2.
\end{equation}
Thus unequal gradient norms give a strictly positive leading variance reduction. The common direction term cancels from the comparison: sharing a direction does not introduce an additional variance penalty for SAZO relative to MeZO. For the learned allocation, the leading gain equals the oracle gain minus $(d+2)\widetilde{\rho}_t/B$, directly connecting the benefit to the tracking analysis.

Larger batches reduce both methods' sampling variance by $1/B$, while norm-proportional allocation minimizes this term for every batch size. SAZO exploits sample heterogeneity within the same two perturbed model states used by MeZO, and duplicate evaluation reuse reduces the work needed for its $B$ draws.

\subsection{Consequences for tracking and convergence}\label{app:batch-tracking}
Because an included example updates its EMA once, its refresh event is $O_{i,t}=\mathbf1\{K_{i,t}>0\}$. Conditional on the history, the probability that all $B$ independent draws miss example $i$ is $(1-p_{i,t})^B$. Thus its inclusion probability and the resulting uniform refresh bound are
\begin{equation}\label{eq:batch-refresh}
 \pi_{i,t}=1-(1-p_{i,t})^B,
 \qquad r_B=\beta\left[1-\left(1-\frac\varepsilon N\right)^B\right].
\end{equation}
The floor $p_{i,t}\ge\varepsilon/N$ gives $\beta\pi_{i,t}\ge r_B$. Relative to the $B=1$ rate $r=\beta\varepsilon/N$, one has $r\le r_B\le\min\{Br,\beta\}$ and $r_B=Br[1+O(B\varepsilon/N)]$ when $B\varepsilon/N$ is small.

To connect inclusion to tracking contraction, set $e_{i,t}=m_{i,t}-a_{i,t}$ and use the leading Gaussian observation $Y_{i,t}^0$ from Appendix~\ref{app:proofs}. Independence of the index draws and the fresh direction gives
\[
\E_t\!\left[\bigl(e_{i,t}+\beta O_{i,t}(Y_{i,t}^0-m_{i,t})\bigr)^2\right]
\le(1-r_B)e_{i,t}^2+\beta^2\pi_{i,t}\sigma_G^2a_{i,t}^2,
\]
since $(2\beta-\beta^2)\pi_{i,t}\ge\beta\pi_{i,t}\ge r_B$. Sharing $z_t$ correlates the observations across examples, but Lemma~\ref{lem:sparse} controls a sum of coordinatewise squared errors and does not require independence across~examples. Its argument therefore applies with $r=r_B$; the moving target and finite-difference remainder are included in Theorem~\ref{thm:finite}. Each included example receives only one directional observation, and updating its EMA once avoids treating duplicate draws as fresh evidence.

To see how the shared-direction term enters optimization, average Eq.~\eqref{eq:batch-second} in the smoothness descent inequality. With $C_d=L_F(d+2)$ and $\Psi_T=T^{-1}\sum_t\E\norm{\nabla F(x_t)}^2$, this gives
\begin{equation}\label{eq:batch-descent}
 \left[1-\frac{C_d\eta}{2}\left(1-\frac1B\right)\right]\Psi_T
 \le\frac{\Delta_F}{\eta T}
 +\frac{C_d\eta}{2B}\left(\widetilde{\D}_T+\frac{\widetilde{\Reg}_T}{T}\right)+O(\mu^2).
\end{equation}
For $C_d\eta(1-1/B)<2$, the coefficient on the left is positive. This coefficient is common to SAZO and a uniform shared-direction update. Along a given parameter trajectory, uniform sampling contributes $\widetilde{\U}_T/B$ to the right-hand sampling term, whereas SAZO contributes $(\widetilde{\D}_T+\widetilde{\Reg}_T/T)/B$. The same dynamic-oracle improvement established in the main analysis therefore persists in the minibatch bound. Setting $B=1$ recovers the main descent lemma.

The same distinction matters for target drift. The coarse moment bound with an $O(\mu^4)$ remainder gives
\[
 Q_T\le2N\Lambda_d\eta^2\left(T\widetilde{\D}_T+\frac{\widetilde{\Reg}_T}{B}\right)+O(T\mu^4),
\]
using $\norm{\bar g_t}^2\le \bar a_t^2$. Combining this inequality with the $12Q_T/r_B^2$ term in Lemma~\ref{lem:sparse} and the factor $8/(N\varepsilon)$ in Lemma~\ref{lem:normalization}, the coefficient of regret on the right-hand side is
\begin{equation}\label{eq:batch-coupling}
\kappa_B=\frac{192\Lambda_d\eta^2}{B\varepsilon r_B^2},
\qquad
\frac{\kappa_B}{\kappa}
=\frac1B\left[\frac{\varepsilon/N}{1-(1-\varepsilon/N)^B}\right]^2.
\end{equation}
Here $\kappa$ is the $B=1$ coefficient at the same $\eta,\varepsilon,\beta$. Thus $\kappa_B/\kappa=B^{-3}[1+O(B\varepsilon/N)]$ when $B\varepsilon/N$ is small: sampling reduces the regret-dependent update moment by $1/B$, and the increased refresh rate reduces its tracking contribution by a further factor of approximately $B^2$. Theorem~\ref{thm:finite}, with $r=r_B$ and $\kappa_\mu=\kappa_B$, combines these effects with the remaining noise, exploration, and finite-difference terms. Its $B=1$ specialization recovers the main convergence analysis.

\section{Experimental settings}\label{app:experimental-settings}
\subsection{Detailed experiments on RoBERTa-large}\label{app:roberta-intro}
We use the hyperparameters in Table~\ref{tab:roberta-hyperparameters} for SAZO experiments on RoBERTa-large. Regarding learning-rate scheduling and early stopping, we use a constant learning rate for all SAZO experiments. All SAZO experiments use $\beta=0.05$, $\varepsilon=0.1$, and four burn-in passes. For Figure~\ref{fig:intro-loss}(a), we select the optimal constant learning rate for each method from $\{1,5\}\times\{10^{-7},10^{-6},10^{-5}\}$; all other settings match the corresponding full-parameter experiments in Table~\ref{tab:main}.

\begin{table}[!htbp]
\centering\small
\caption{Hyperparameter grids used for SAZO experiments on RoBERTa-large.}\label{tab:roberta-hyperparameters}
\begin{tabular}{lll}
\toprule
Experiment & Hyperparameter & Values\\
\midrule
\multirow{4}{*}{SAZO}
  & Batch size & 64\\
  & Learning rate & $\{10^{-7},10^{-6},10^{-5}\}$\\
  & Perturbation radius $\mu$ & $10^{-3}$\\
  & Weight decay & 0\\
\midrule
\multirow{5}{*}{SAZO (prefix)}
  & Batch size & 64\\
  & Learning rate & $\{10^{-2},5\times10^{-3},10^{-3}\}$\\
  & Perturbation radius $\mu$ & $10^{-1}$\\
  & Weight decay & 0\\
  & Prefix tokens & 5\\
\midrule
\multirow{5}{*}{SAZO (LoRA)}
  & Batch size & 64\\
  & Learning rate & $\{10^{-5},5\times10^{-5},10^{-4}\}$\\
  & Perturbation radius $\mu$ & $10^{-3}$\\
  & Weight decay & 0.1\\
  & $(r,\alpha)$ & $(8,16)$\\
\bottomrule
\end{tabular}
\end{table}

\subsection{Detailed results on various LLMs}\label{app:extended-models}\label{app:llama-intro}
We use the hyperparameters in Table~\ref{tab:glm-hyperparameters} for SAZO experiments on GLMs. For Figure~\ref{fig:intro-loss}(b,c), we select the optimal constant learning rate for each method and model from $\{1,5\}\times\{10^{-7},10^{-6},10^{-5}\}$; all other settings match the corresponding experiments in Table~\ref{tab:opt13b}.

\begin{table}[!htbp]
\centering\small
\caption{Hyperparameter grids used for GLM experiments. All weight decay is set to 0. SAZO uses 20K steps and constant learning rates.}\label{tab:glm-hyperparameters}
\begin{tabular}{lll}
\toprule
Experiment & Hyperparameter & Values\\
\midrule
\multirow{3}{*}{SAZO}
  & Batch size & 16\\
  & Learning rate & $\{10^{-6},5\times10^{-7},10^{-7}\}$\\
  & Perturbation radius $\mu$ & $10^{-3}$\\
\bottomrule
\end{tabular}
\end{table}
\FloatBarrier

\subsection{Training-population scaling settings}\label{app:population-scaling}
Figure~\ref{fig:population-scaling} uses full-parameter RoBERTa-large on SST-2 with training populations $N=1{,}024$, $2N=2{,}048$, $3N=3{,}072$, and $4N=4{,}096$. MeZO and SAZO use batch size 16, constant learning rate $10^{-6}$, and 100K optimization steps. All learning-rate and checkpoint choices are fixed or validation-only; the test set is never used for selection. We report test accuracy averaged over seeds 13, 21, and 42, with population standard deviations as error bars.

\subsection{Ablation experiment details}\label{app:hyperparameter-ablation}
The ablations use one seed of full-parameter RoBERTa-large on SST-2 with $k=512$, a fixed 1,024-example training population, batch size 64, constant learning rate $3\times10^{-6}$, perturbation radius $10^{-3}$, and zero weight decay. We initialize the importance scores by averaging directional magnitudes over four passes, then use sampling with replacement and inverse-probability-weighted parameter updates. The $\beta$ sweep in Figure~\ref{fig:hyperparameter-ablation} uses $\beta\in\{0.005,0.05,0.5,1\}$ at fixed $\varepsilon=0.1$. The exploration sweep uses $\varepsilon\in\{0.05,0.1,0.5,1\}$ at fixed $\beta=0.05$. Each plotted point is a forward-only mean-loss audit over the complete training population; audit queries are excluded from the update count. We retain every measured audit and apply neither curve smoothing nor extrapolation.

Figure~\ref{fig:epsilon-ablation} shows stable optimization across $\varepsilon\in\{0.05,0.1,0.5\}$, indicating that SAZO is robust to substantial changes in the exploration rate. At the extreme $\varepsilon=1$, the adaptive component vanishes and sampling becomes uniform; this trajectory eventually diverges under the same learning rate.

\begin{figure}[!htbp]
\centering
\includegraphics[width=.78\linewidth]{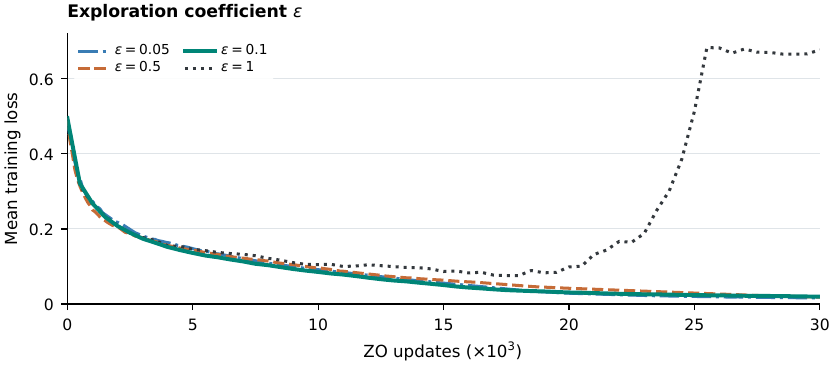}
\caption{Exploration-coefficient ablation on SST-2. SAZO remains stable over $\varepsilon\in\{0.05,0.1,0.5\}$, whereas $\varepsilon=1$, which reduces to uniform sampling, becomes unstable.}
\label{fig:epsilon-ablation}
\end{figure}
\FloatBarrier

\subsection{Sample-importance diagnostics}\label{app:heterogeneity}
The heterogeneity experiments use full-parameter RoBERTa-large under uniform MeZO training, with $k=512$, batch size 64, learning rate $10^{-6}$, perturbation radius $10^{-3}$, and zero weight decay. At frozen checkpoints, we compute exact per-example gradient norms over the trainable parameters. These gradients are used only for diagnostics. The oracle allocation is $p_{i,t}^\star=a_{i,t}/\sum_j a_{j,t}$, and relative importance is $Np_{i,t}^\star$. Top-10\% sets are recomputed at each checkpoint; temporal changes are measured using Spearman rank correlation. For the SST-2 examples, the prompt appends \texttt{It was <mask>.}, with \texttt{great} and \texttt{terrible} as the positive and negative verbalizers.

Figure~\ref{fig:intro-narrative}(a) reports the exact audits at initialization, 10K, and 30K updates. We divide all three checkpoints by the median initial norm and fix the horizontal order using the initial norms, so a given position refers to the same training example at every checkpoint. Every example is shown as a point, without smoothing or aggregation.

\begin{figure}[!htbp]
\centering
\includegraphics[width=.65\linewidth]{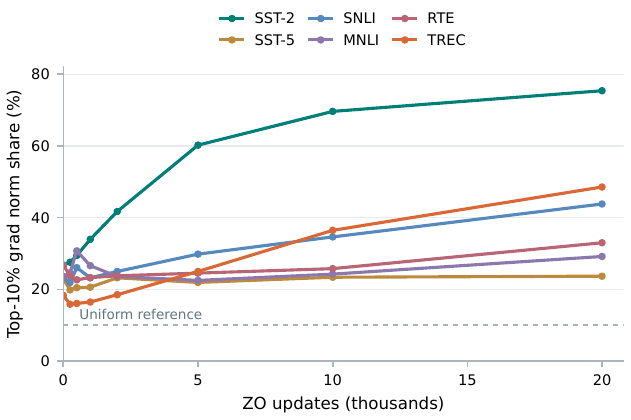}
\caption{\textbf{Gradient-norm concentration across tasks.} The top 10\% of examples account for an increasing share of the total gradient norm across six tasks. The top set is recomputed at each checkpoint; the dashed line marks the approximately 10\% uniform reference. Exact gradients are measured at frozen RoBERTa-large MeZO checkpoints.}\label{fig:multitask-concentration}
\end{figure}

\paragraph{Length and label controls.}\label{app:heterogeneity-controls}
We group examples by label or by the joint pair of prompt token count and label, and optimize the allocation subject to equal probabilities within each group. For a partition $\mathcal G$, the optimal group-restricted distribution is
\[
 b_{G,t}=\left(\frac{1}{|G|}\sum_{i\in G}a_{i,t}^2\right)^{1/2},
 \qquad
 p_{i,t}^{\mathcal G}=\frac{b_{G(i),t}}{\sum_{G\in\mathcal G}|G|b_{G,t}}.
\]
The fraction of the uniform-to-oracle variance gap remaining is
$[\widetilde V_t(p_t^{\mathcal G})-\widetilde V_t(p_t^\star)]/[\widetilde V_t(\mathbf1/N)-\widetilde V_t(p_t^\star)]$, evaluated at 20K updates.

\paragraph{Online versus fixed sampling.}\label{app:regret-settings}
The SST-2 comparison uses a common initial model, batch size 64, perturbation radius $10^{-3}$, and 30K updates. We scan constant learning rates in $\{10^{-6},3\times10^{-6}\}$. SAZO uses $\beta\in\{0.025,0.05,0.075\}$ and $\varepsilon=0.1$. The fixed sampler retains the initial exact gradient-norm-proportional distribution $p_0^\star$ without an exploration mixture; uniform sampling uses $1/N$. At each audit, the normalized instantaneous regret is
\begin{equation}\label{eq:empirical-regret}
 r_t=\frac{\widetilde V_t(p_t)-\widetilde V_t(p_t^\star)}{\widetilde V_t(p_t^\star)}.
\end{equation}
Figure~\ref{fig:intro-narrative}(b) reports relative cumulative regret,
\begin{equation}\label{eq:empirical-cumulative-regret}
 \frac{\int_0^T [\widetilde V_t(p_t)-\widetilde V_t(p_t^\star)]\,dt}
 {\int_0^T \widetilde V_t(p_t^\star)\,dt}.
\end{equation}
We integrate the numerator and denominator separately using the trapezoidal rule on all 61 recorded audits, spaced 500 updates apart from 0 to 30K, and then take their ratio. This differs from integrating $r_t$. The undefined ratio at $T=0$ is omitted from the plot. Exact gradients at each method's current parameters define its reference oracle. Loss is evaluated over the complete training population.

\input{sections/experiment_details}

%% file: sections/theory_details.tex
\subsection{Supporting moment and descent bounds}
The following results use the normalized quantities defined above. Multiplying the normalized regret bound by $d+2$ gives Theorem~\ref{thm:tracking}; substituting $C_d=L_F(d+2)$ and $\widetilde\D_T=\D_T/(d+2)$ gives Theorem~\ref{cor:rate}.

\begin{proposition}[Feedback and update moments]\label{prop:moments}
For $\sigma_G^2=\pi/2-1$, the finite-difference observations satisfy
\begin{equation}\label{eq:obs-moments}
\begin{gathered}
\E_tY_{i,t}=a_{i,t}+O(\mu^2),\\[3pt]
\operatorname{Var}_t(Y_{i,t})=\sigma_G^2a_{i,t}^2+O(\mu^2).
\end{gathered}
\end{equation}
The corresponding update satisfies
\begin{equation}\label{eq:update-moments}
\begin{gathered}
\norm{\E_t\widehat g_t-\nabla F(x_t)}=O(\mu^2),\\[3pt]
\E_t\norm{\widehat g_t}^2\le2(d+2)\widetilde V_t(p_t)+O(\mu^4).
\end{gathered}
\end{equation}
\end{proposition}

\begin{lemma}[Descent decomposition]\label{thm:outer}
Under Assumption~\ref{ass:smooth}, let $\Psi_T=T^{-1}\sum_t\E\norm{\nabla F(x_t)}^2$.~Then
\begin{equation}\label{eq:outer}
\Psi_T\le\frac{\Delta_F}{\eta T}
+\frac{C_d\eta}{2}\widetilde{\D}_T
+\frac{C_d\eta}{2}\frac{\widetilde{\Reg}_T}{T}+O(\mu^2),
\end{equation}
where $C_d=L_F(d+2)$ and $\Delta_F=F(x_1)-F_*$.
\end{lemma}

\paragraph{Schedules and finite-difference accuracy.}
The schedules in Theorem~\ref{cor:rate} are constant within each run and indexed by the horizon. The leading-term choice of $c_\eta$ requires $\D_T$ and $\Delta_F$; since $\D_T$ depends on the resulting trajectory, this is a conditional tuning result. The finite-time bounds retain the trajectory moments without imposing a global gradient bound, convexity, or a strong-growth condition. For fixed problem and schedule coefficients and uniformly bounded trajectory moments, the $O(\mu^2)$ constant can be taken uniformly over feasible horizons. A run-wise radius $\mu=\mu_0T^{-\alpha}$ with $\alpha>3/10$ makes this correction $o(T^{-3/5})$; a fixed radius retains an additive error floor. Appendix~\ref{app:fd} gives the explicit bounds.

%% file: sections/experiment_details.tex
\subsection{Interpretable sample-level structure}\label{app:additional-evidence}\label{app:sample-cases}
Figures~\ref{fig:importance-dynamics} and~\ref{fig:sample-importance-examples} connect population-level changes in sample importance to the model's predictions on individual inputs. Importance is measured by the normalized exact gradient norm $N\norm{\nabla f_i(x_t)}/\sum_j\norm{\nabla f_j(x_t)}$ at frozen RoBERTa-large MeZO checkpoints; a value of 1 corresponds to uniform allocation. These are diagnostic oracle quantities, rather than SAZO's estimated scores. Detailed audit settings are given in Appendix~\ref{app:heterogeneity}.

\paragraph{Sample rankings change during training.}
Figure~\ref{fig:importance-dynamics} compares each example's importance at initialization and after 20K updates on SST-2 and TREC. Points above the diagonal gain relative importance, while those below lose it. The initial-to-20K Spearman correlations are 0.52 and 0.12, respectively, showing that the initial ranking does not reliably preserve subsequent allocation priorities.

\begin{figure}[H]
\centering
\includegraphics[width=.78\linewidth]{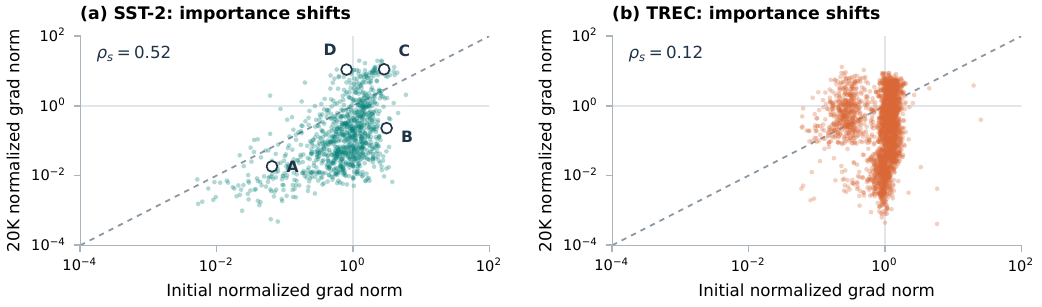}
\caption{\textbf{Sample-importance rankings evolve during training.} Initial versus 20K normalized gradient norms $N\norm{\nabla f_i(x_t)}/\sum_j\norm{\nabla f_j(x_t)}$ on SST-2 and TREC, with one point per example and shared logarithmic axes. The diagonal marks unchanged importance, and $\rho_s$ is the Spearman correlation. A--D identify the SST-2 examples examined in Figure~\ref{fig:sample-importance-examples}. All gradient norms are exact audits at frozen MeZO checkpoints.}\label{fig:importance-dynamics}
\end{figure}

\paragraph{Importance changes have interpretable sample-level structure.}
The four SST-2 cases A--D in Figure~\ref{fig:sample-importance-examples} link the scatter plot to predictions at initialization, 5K, and 20K updates. A (``a great cast'') stays correct and has low importance throughout. B (``over-dramatic at times'') becomes correct as its normalized importance falls from 3.113 to 0.230. C (``smug grin'') remains incorrect, while its importance rises from 2.863 to 11.170. D, containing both ``amateurish'' and ``a minor treat,'' becomes incorrect and rises from 0.805 to 10.930. For D, the absolute gradient norm also increases, so the rise is not solely due to other examples' shrinking gradients.

\begin{figure}[H]
\centering
\includegraphics[width=.78\linewidth]{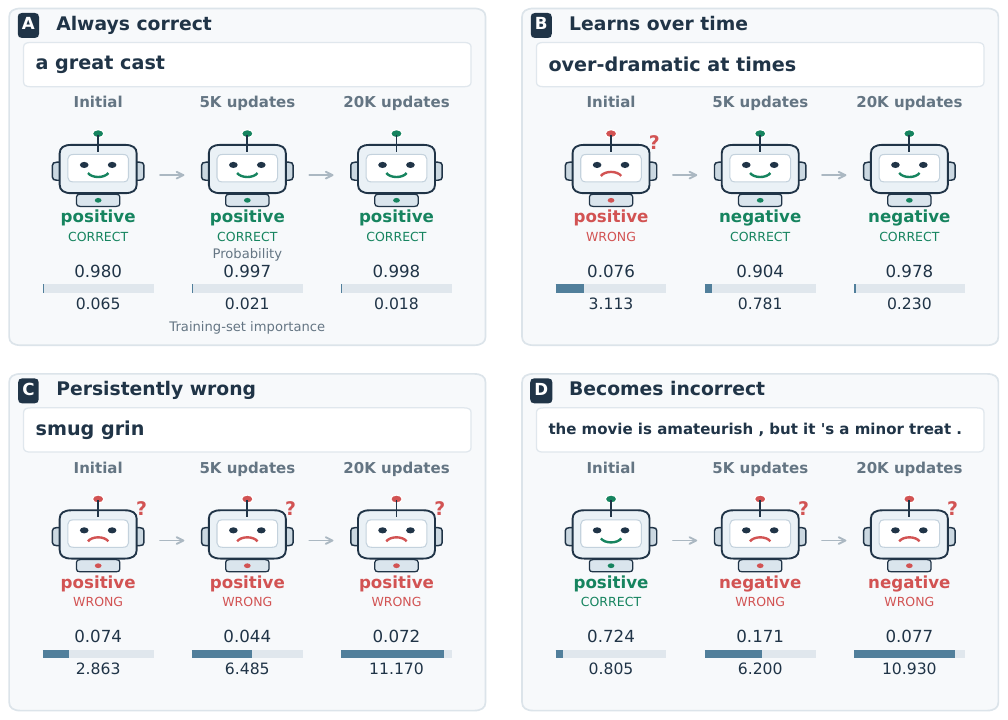}
\caption{\textbf{Representative examples reveal interpretable importance changes.} A--D correspond to the SST-2 points marked in Figure~\ref{fig:importance-dynamics}. Robot lights and status labels indicate correctness at each checkpoint. Numbers above the blue bars give the correct-class probability $\exp(-f_i(x_t))$; numbers below give normalized gradient norms. Bars share a linear scale from 0 to 12, with 1 corresponding to uniform allocation. All gradient norms are exact audits at frozen MeZO checkpoints. The input template appends \texttt{It was <mask>.} Appendix~\ref{app:heterogeneity} provides detailed settings and sample~identities.}\label{fig:sample-importance-examples}
\end{figure}

Together, the population shifts and these illustrative cases show that sample importance reflects the model's evolving response to an input. Words, length, and labels remain fixed, yet predictions and gradient contributions change. An allocation based only on initial importance can therefore become stale, motivating continued tracking during~optimization.
\FloatBarrier